\documentclass[twoside,11pt, letter]{article}

\usepackage{fixltx2e} 
\usepackage{cmap}
\usepackage[T1]{fontenc}
\usepackage{lmodern}
\usepackage[utf8]{inputenc}
\usepackage{verbatim}
\usepackage{booktabs}       
\usepackage{nicefrac}       
\usepackage{microtype}      
\usepackage{graphicx} 
\usepackage[colorlinks]{hyperref}
\usepackage{url}
\usepackage{geometry} % to change the page dimensions
\usepackage[toc,page]{appendix}
\usepackage{mathrsfs}
\usepackage{natbib}

\bibpunct{(}{)}{;}{a}{,}{,}

\usepackage[linesnumbered, ruled]{algorithm2e}
\usepackage{tikz}
\usepackage{float}
\usepackage{bm}
\usepackage{afterpage}

\usepackage[font=small,labelfont=bf,labelsep=period]{caption}

\usepackage{amsmath, amsbsy, amsgen, amscd, amsthm, amsfonts, amssymb} 

\providecommand{\mathbold}[1]{\bm{#1}}

\newtheoremstyle{myThm}   % Name
     {\topsep}                          % Space above 
     {\topsep}                          % Space below
     {\itshape}                         % Body font
     {}                                      % Indent amount
     {\sffamily\bfseries}           % Theorem head font
     {.}                                     % Punctuation after theorem head
     {.5em}                              % Space after theorem head
     {}                                      % Theorem head specifications. Empty means normal

\newtheoremstyle{myRem}   % Name
     {\topsep}                          % Space above 
     {\topsep}                          % Space below
     {}                         % Body font
     {}                                      % Indent amount
     {\sffamily\bfseries}           % Theorem head font
     {.}                               % Punctuation after theorem head
     {.5em}                              % Space after theorem head
     {}                                      % Theorem head specifications. Empty means normal
     
\newtheoremstyle{myDef}   % Name
     {\topsep}                          % Space above 
     {\topsep}                          % Space below
     {\itshape}                         % Body font
     {}                                      % Indent amount
     {\sffamily\bfseries}           % Theorem head font
     {.}                                 % Punctuation after theorem head
     {.5em}                              % Space after theorem head
     {}                                      % Theorem head specifications. Empty means normal     

\theoremstyle{myThm}

\newtheorem{lemma}{Lemma}
\newtheorem{proposition}{Proposition}
\newtheorem{corollary}{Corollary}

\theoremstyle{myRem}
\newtheorem{remark}{Remark}

\theoremstyle{myDef}
\newtheorem{definition}{Definition}

\usepackage[toc,page]{appendix}
\usepackage{makeidx}
\makeindex
\usepackage{float}
\usepackage{mathrsfs}
\usepackage{framed}
\usepackage{float}
\usepackage[font=small,labelfont=bf,labelsep=period]{caption}

\usepackage{graphicx} 
\usepackage{epstopdf}

\newcommand{\couplings}{\Gamma}

\newcommand{\diag}{\text{diag}}
\newcommand{\dist}{P}
\newcommand{\distset}{\mathcal{P}}

\newcommand{\emprad}{\hat{\mathfrak{R}}}

\renewcommand{\epsilon}{\varepsilon}
\DeclareMathOperator*{\essup}{\text{ess sup}}
\newcommand{\functions}{\mathcal{F}}
\newcommand{\lhinge}{\ell_{\text{h}}}
\newcommand{\lhzo}{\ell_{\text{h}, 01}}

\newcommand{\linear}{\text{lin}}
\newcommand{\lmargin}{\ell_{\rho}}
\newcommand{\lr}{\ell_{r}}
\newcommand{\lrminus}{\ell_{r, -}}
\newcommand{\lrplus}{\ell_{r, +}}
\newcommand{\lrtrunc}{\ell_{r, B}}
\newcommand{\lrtruncpn}{\ell_{r, B}^{\pm}}
\newcommand{\lxe}{\ell_{\text{xe}}}
\newcommand{\lzo}{\ell_{01}}

\newcommand{\metric}{d}
\newcommand{\nets}{\text{nn}}
\newcommand{\ntrans}{\Psi_{-}}

\newcommand{\ptrans}{\Psi_{+}}

\newcommand{\rademacher}{\mathfrak{R}}
\newcommand{\reals}{\mathbb{R}}
\newcommand{\risk}{R}
\newcommand{\robust}{\text{rob}}
\newcommand{\sfield}{\mathcal{G}}
\newcommand{\tree}{T}
\newcommand{\treepos}{\tree_{+}}
\newcommand{\treeneg}{\tree_{-}}
\newcommand{\trans}{\intercal}

\newcommand{\wdist}{W}
\renewcommand{\xspace}{\mathcal{X}}
\newcommand{\yspace}{\mathcal{Y}}
\newcommand{\zspace}{\mathcal{Z}}
\newcommand{\ellbar}{\ensuremath{\bar{\ell}}}
\newcommand{\scriptD}{\ensuremath{\mathcal{D}}}
\newcommand{\sdp}{\text{SDP}}
\newcommand{\real}{\ensuremath{\mathbb{R}}}

\newcommand{\expect}{\mathbb{E}}
\newcommand{\prob}{\mathbb{P}}

\newcommand{\sgn}{\text{sgn}}
\newcommand{\ind}{\mathbold{1}}

\title{Adversarial risk bounds via function transformation}

\author{Justin Khim\thanks{Department of Statistics, University of Pennsylvania, Philadelphia, PA 19104.}
\and 
Po-Ling Loh\thanks{Department of Electrical \& Computer Engineering and Statistics,
       University of Wisconsin,
       Madison, WI 53706.}}

\date{\today}
\begin{document}

\maketitle

%---------------------------------------------%
%---------------------------------------------%

\begin{abstract}
We derive bounds for a notion of adversarial risk, designed to characterize the robustness of linear and neural network classifiers to adversarial perturbations. Specifically, we introduce a new class of function transformations with the property that the risk of the transformed functions upper-bounds the adversarial risk of the original functions. This reduces the problem of deriving bounds on the adversarial risk to the problem of deriving risk bounds using standard learning-theoretic techniques. We then derive bounds on the Rademacher complexities of the transformed function classes, obtaining error rates on the same order as the generalization error of the original function classes. We also discuss extensions of our theory to multiclass classification and regression. Finally, we provide two algorithms for optimizing the adversarial risk bounds in the linear case, and discuss connections to regularization and distributional robustness.
\end{abstract}

%---------------------------------------------%
%---------------------------------------------%
\section{Introduction}
\label{secIntro}
Deep learning systems are becoming ubiquitous in everyday life. From virtual assistants on phones to image search and translation, neural networks have vastly improved the performance of many computerized systems in a short amount of time \citep{goodfellow2016}. However, neural networks have a variety of shortcomings: A peculiarity that has gained much attention over the past few years has been the apparent lack of robustness of neural network classifiers to adversarial perturbations. \cite{szegedy2013} noticed that small perturbations to images could cause neural network classifiers to predict the wrong class. Furthermore, these perturbations could be carefully chosen so as to be imperceptible to the human eye.

Such observations have instigated a deluge of research in finding adversarial attacks \citep{athalye2018, goodfellow2014, papernot2016, szegedy2013}, defenses against adversaries for neural networks \citep{madry2018, raghunathan2018, sinha2018, wong2018}, evidence that adversarial examples are inevitable \citep{shafahi2018}, arguments that robust learning requires more data \citep{schmidt2018}, and theory suggesting that constructing robust classifiers is computationally infeasible \citep{bubeck2018}. Attacks are usually constructed assuming a white-box framework, in which the adversary has access to the network, and adversarial examples are generated using a perturbation roughly in the direction of the gradient of the loss function with respect to a training data point.  This idea generally produces adversarial examples that can break ad-hoc defenses in image classification, and some work exists on extending attacks even to the black-box setting \citep{ilyas2018}.

Currently, strategies for creating robust classification algorithms are much more limited. One approach \citep{madry2018, suggala2018} is to formalize the problem of robustifying a network as a novel optimization problem, where the objective function is the expected loss of a supremum over possible perturbations. However, \cite{madry2018} note that the objective function is often not concave with respect to the perturbation. Other authors \citep{raghunathan2018, wong2018} have leveraged convex relaxations to provide optimization-based certificates on the adversarial loss of the training data. However, the generalization performance of the training error to unseen examples is not considered. 

The optimization community has long been interested in constructing robust solutions for various problems, such as portfolio management \citep{ben2009}, and deriving theoretical guarantees. Robust optimization has been studied in the context of regression and classification \citep{trafalis2007, xu2009lasso, xu2009svm}. More recently, a notion of robustness that attempts to minimize the risk with respect to the worst-case distribution close to the empirical distribution has been the subject of extensive work \citep{ben2013, namkoong2016, namkoong2017}. Researchers have also considered a formulation known as distributionally robust optimization, using the Wasserstein distance as a metric between distributions \citep{esfahani2015, BlaKan17, gao2017, sinha2018}. With the exception of \cite{sinha2018}, however, generalization bounds of a learning-theoretic nature are nonexistent, with most papers focusing on studying properties of a regularized reformulation of the problem. \cite{sinha2018} provide bounds for Wasserstein distributionally robust generalization error based on covering numbers for sufficiently small perturbations. This is sufficient for ensuring a small amount of adversarial robustness and is quite general. To instantiate these bounds, one could then use a covering number \citep{bartlett2017} or Rademacher complexity bound \citep{golowich2018}.

Although neural networks are rightly the subject of attention due to their ubiquity and utility, the theory that has been developed to explain the phenomena arising from adversarial examples is still far from complete. For example, \cite{goodfellow2014} argue that non-robustness may be due to the linear nature of neural networks. However, attempts at understanding linear classifiers \citep{fawzi2018} argue \emph{against} linearity, i.e., the function classes should be more expressive than linear classification.

In this paper, we provide upper bounds for a notion of adversarial risk in the case of linear predictors and neural networks. These bounds may be viewed as a sample-based guarantee on the risk of a trained predictor, even in the presence of adversarial perturbations on the inputs. The key step is to transform a predictor \(f\) into an ``adversarially perturbed'' predictor \(\Phi f\) by modifying the loss function. The risk of the function \(\Phi f\) can then be analyzed in place of the adversarial risk of \(f\); in particular, we can more easily provide bounds  on the Rademacher complexities necessary for bounding the robust risk. Finally, our transformations suggest algorithms for minimizing the adversarially robust empirical risk. Thus, as a consequence of the theory developed in this paper, we can show that adversarial perturbations have somewhat limited effects from the point of view of generalization error.

In concurrent work, \cite{yin2018} study the generalization error for binary and multiclass classification by utilizing the method of \cite{raghunathan2018}, providing the first bounds on generalization error in the multiclass case.
These upper bounds are not comparable in general, which we discuss in more detail in Appendix~\ref{appIncomparability}.
%\textcolor{red}{Add more details. Are the bounds in the two papers completely incomparable?} We employ the same techniques used in \cite{yin2018} to extend our results for binary classification to the multiclass setting, using a multiclass extension of Rademacher complexity bounds \citep{mohri2012}. 
Additionally, whereas the scope of \cite{yin2018} and \cite{raghunathan2018} is limited to neural networks with one hidden layer and ReLU activation functions, our approach is applicable to a broader class of neural networks. On a technical side, we also upper-bound the resulting adversarial loss in different ways: \cite{yin2018} use covering number bounds developed in \cite{bartlett2017}, whereas we have used Rademacher complexity techniques from \cite{golowich2018}. It is possible that a similar covering number analysis could be given, which is an avenue for future work.

This paper is organized as follows: We introduce the precise mathematical framework in Section~\ref{secSetup}. In Section~\ref{secMain}, we discuss our main results for binary classification. 
In Sections~\ref{secMulticlass} and \ref{secRegression}, we extend our results to multiclass classification and regression, respectively.
In Section~\ref{subsecAlgorithms}, we provide results on optimizing the adversarial risk bounds in the linear case and discuss computational considerations for neural networks. In Section~\ref{secProofs}, we prove our key theoretical contributions in the case of binary classification. 
% In Section~\ref{secNumerical}, we provide numerical results using the MNIST dataset. 
Finally, we conclude with a discussion of future avenues of research in Section~\ref{secDiscussion}. Additional proof details are contained in the appendices.

\textbf{Notation:} For a matrix $A \in \real^{n \times n}$, we write $\|A\|_\infty$ to denote the $\ell_\infty$-operator norm. We write $\|A\|_F$ to denote the Frobenius norm. We also define the matrix \((p, q)\)-norm by
\[
\|A\|_{p, q}
=
\left(\sum_{k = 1}^{K} \|a_i\|_{p}^{q}\right)^{\frac{1}{q}},
\]
where $A = (a_1, \dots, a_n)^\trans$. We write $\|v\|_p$ to denote the $\ell_p$-norm of a vector $v \in \real^n$. For two vectors $v,w \in \real^n$, we use $v \circ w$ to denote the vector with $i^\text{th}$ component equal to $v_i w_i$. We write $\real^n_+$ to denote the set $\left\{v \in \real^n: v_i \ge 0 \quad \forall 1 \le i \le n \right\}$.

%---------------------------------------------%
%---------------------------------------------%
\section{Problem setup}
\label{secSetup}

We consider a standard statistical learning setup.
For simplicity, we first establish the notation for binary classification and later comment on appropriate adjustments to the labels, losses, and function classes for multiclass classification and regression.
Let \(\xspace \subseteq \reals^{m}\) be a space of covariates, and define the space of labels to be \(\yspace = \{+1, -1\}\). Let \(\zspace = \xspace \times \yspace\). Suppose we have \(n\) observations \(z_{1} = (x_{1}, y_{1}), \ldots, z_{n} = (x_{n}, y_{n})\), drawn i.i.d.\ according to some unknown distribution \(\dist\). We write \(S = \{z_1, \dots, z_n\}\).

A classifier corresponds to a function \(f: \xspace \to \mathcal{D}\), where \(\yspace \subseteq \mathcal{D}\). Thus, the function \(f\) may express uncertainty in its decision; e.g., prediction in \(\mathcal{D} = [-1, +1]\) allows the classifier to select an expected outcome.

%---------------------------------------------%
%---------------------------------------------%
\subsection{Risk and losses}
\label{subsecRiskLoss}

Given a loss function \(\ell: \mathcal{D} \times \zspace \to \reals^{+}\),
our goal is to minimize the \emph{adversarially robust risk}, defined by
\begin{equation*}
\risk_{\robust}(\ell, f) = \expect_{z \sim \dist}\left[\sup_{w \in B(\epsilon)}\ell(f, z + w)\right],
\label{eqnRobustRisk}
\end{equation*}
%where we write $\ell(f,z) = \ell(f(x),y)$, and
where \(w\) is an adversarially chosen perturbation in the \(\ell_{p}\)-ball \(B(\epsilon) \subseteq \reals^{m}\) of radius \(\epsilon\). For simplicity, we write \(z + w = (x + w, y)\), so the input is perturbed by a vector in the $\ell_p$-ball of radius $\epsilon$, but is still classified according to $f(x)$. A popular choice of \(p\) in the literature is \(1\), \(2\), or \(\infty\); the case \(p = \infty\) has received particular interest. Also note that if $\epsilon = 0$, the adversarial risk reduces to the usual statistical risk, for which upper bounds based on the empirical risk are known as generalization error bounds. For some discussion of the relationship between the adversarial risk to the distributionally robust risk, see Appendix~\ref{appAdvDist}.
%so one can take \(p = \infty\) for concreteness.
%We also note that a norm \(\|\cdot \|_{p}\) in this paper refers to an \(\ell_{p}\)-norm rather than an \(L_{p}\)-norm unless otherwise noted.
%Finally, W we write \(\ell(f, w)\) as shorthand for \(\ell(f(x), y)\).

We now define a few specific loss functions. The indicator loss (also known as the 01-loss) is defined by
\[
\lzo(f, z) = \ind\left\{\sgn f(x) = y\right\},
\]
and is of primary interest in classification.
%in both the linear classifier and neural network classification settings, we will primarily be interested in bounding the adversarial risk with respect to the indicator loss.
However, it is often difficult to analyze the indicator loss directly, so we instead analyze convex, Lipschitz surrogate functions that upper-bound the indicator loss~\citep{BarEtal06, mohri2012}. Accordingly, we define the hinge loss
\[
\lhinge(f, z) = \max\left\{0, 1 - y f(x)\right\},
\]
which is a convex surrogate for the indicator loss, and will appear in some of our bounds.
We also introduce
%one non-standard ``loss function'' for convenience.
%This is
the indicator of whether the hinge loss is positive, defined by
\[
\lhzo(f, z)
=
\ind\left\{\lhinge(f, z) > 0\right\}.
\]
%This provides a helpful upper bound in some of our analyses, because the indicator loss may be \(0\), while the hinge loss is nonzero. 
%In neural networks, one loss function of interest is
For analyzing neural networks, we will also employ the cross-entropy loss, defined by
\begin{align*}
&\begin{aligned}
\lxe(f, z)
&=
-\left(\frac{1 + y}{2}\right) \log_{2} \left(\frac{1 + \delta(f(x))}{2}\right)
-\left(\frac{1 - y}{2}\right) \log_{2} \left(\frac{1 - \delta(f(x))}{2}\right),
\end{aligned}
\end{align*}
%\textcolor{red}{The sentence below isn't obvious. Either elaborate further or cut it out (and maybe elaborate later, in the proof of Corollary 2?}
% Note that the use of the base \(2\) here is a scaling convenience so that \(\lzo(f, z) \leq \lxe(f, z)\).
%\textcolor{red}{The next few sentences are also confusing to me (are they necessary?).} 
%The input to the cross-entropy loss is generally not \(f(x)\), but rather \(\delta(f(x))\),
where \(\delta\) is the softmax function:
\[
\delta(w) = \frac{\exp(w) - 1}{\exp(w) + 1}.
\]
%For our purposes, we may consider the softmax and cross-entropy loss to be composed, so we define \(\lxe(f, z) := \lxe(\delta(f(x)), y)\), instead.
%Importantly, this function is \(1\)-Lipschitz in \(f(x)\), even though the cross-entropy loss itself is not Lipschitz.
Note that in all of the cases above, we can also write the loss $\ell(f,z) = \ellbar(f(x), y)$, for an appropriately defined loss $\ellbar: \scriptD \rightarrow \real^+$. Furthermore, $\ellbar_h$ and $\ellbar_{\text{xe}}$ are 1-Lipschitz.

%---------------------------------------------%
%---------------------------------------------%
\subsection{Function classes and Rademacher complexity}
\label{subsecFunctionsRademacher}

We are particularly interested in two function classes: 
linear classifiers and neural networks.
We denote the first class by \(\functions_{\linear}\), 
and we write an element \(f\) of \(\functions_{\linear}\), parametrized by $\theta \in \reals^m$ and $b \in \reals$, as
\[
f(x)
=
\theta^{\trans} x + b.
\]
We similarly denote the class of neural networks as \(\functions_{\nets}\), and we write a neural network \(f\), parametrized by $\{A^{(k)}\}$ and $\{s_k\}$, as
\[
f(x) = A^{(d + 1)} s_{d}(A^{(d)} s_{d - 1}(\ldots s_{1}(A^{(1)} x))),
\]
where each \(A^{(k)}\) is a matrix and each \(s_{k}\) is a monotonically increasing \(1\)-Lipschitz activation function applied elementwise to vectors, such that \(s_{k}(0) = 0\). For example, we might have $s_k(u) = \max\{0, u\}$, which is the ReLU function. The matrix \(A^{(k)}\) is of dimension \(J_{k} \times J_{k - 1}\), where \(J_{0} = m\) and $J_{d+1} = 1$. We use \((a^{(k)}_{j})^T\) to denote the \(j^{\text{th}}\) row of \(A^{(k)}\), with \(r^{\text{th}}\) entry \(a^{(k)}_{j, r}\). Also, when discussing indices, we write \(j_{2: d + 1}\) as shorthand for \(j_{2}, \ldots, j_{d + 1}\). 
%Thus, \(A_{d + 1}\) is a row vector
%However, when we write a norm \(\|A_{d + 1}\|\), we mean this as a matrix norm. 
%If instead we write \(a^{(d + 1)}\), then this is treated as a vector, in which case \(\|a^{(d + 1)}\|\) refers to a vector norm.

A standard measure of the complexity of a class of functions is the Rademacher complexity.
The empirical Rademacher complexity of a function class \(\functions\) and a sample \(S\) is
\begin{equation}
\emprad_{n}(\functions)
=
\frac{1}{n}
\expect_{\sigma} 
\left[\sup_{f \in \functions} \sum_{i = 1}^{n} \sigma_{i} f(x_{i})\right],
\label{eqnEmpRad}
\end{equation}
where the \(\sigma_{i}\)'s are i.i.d.\ Rademacher random variables;
i.e., the \(\sigma_{i}\)'s are random variables taking the values \(+1\) and \(-1\), each with probability \(1/2\).
Note that \(\expect_{\sigma}\) denotes the expectation with respect to the \(\sigma_{i}\)'s.
Finally, we note that the standard Rademacher complexity is obtained by taking an expectation over the data: 
\(\rademacher_{n}(\functions) = \expect \left[ \emprad_{n}(\functions)\right]\).

%---------------------------------------------%
%---------------------------------------------%
\section{Main results}
\label{secMain}

We introduce our main results for binary classification in this section. The trick is to push the supremum through the loss and incorporate it into the function \(f\), yielding a transformed function \(\Phi f\).
We require this transformation to satisfy
\[
\sup_{w \in B(\epsilon)} \ell(f, z + w) \leq \ell(\Phi f, z),
\]
so that an upper bound on the transformed risk will lead to an upper bound on the adversarial risk. We call the proposed functions $\Phi$ the \emph{supremum transformation} and \emph{tree transformation} in the cases of linear classifiers and neural networks, respectively.

In both cases, we have to make a minor assumption about the monotonicity of the loss. We state this as a definition:

\begin{definition}
\label{DefMonotone}
We say that $\ell(f,z) = \ellbar(f(x), y)$ is a \emph{monotonic loss} if \(\ell(f, z)\) is monotonically decreasing in \(yf(x)\): Specifically, \(\ellbar(f(x), +1)\) is decreasing in \(f(x)\) and \(\ellbar(f(x), -1)\) is increasing in \(f(x)\).
\end{definition}
It is easy to verify that all the losses mentioned earlier satisfy the monotonicity property. One technicality is that the transformed function \(\Phi f\) needs to be a function of both \(x\) and \(y\); i.e., we have \(\Phi f: \xspace \times \yspace \to \mathcal{D}\). Thus, the loss of a transformed function is \(\ell(\Phi f, z) = \ell(\Phi f(x, y), y)\).
%instead of \(\ell(f, z) = \ell(f(x), y)\).
We now define the essential transformations studied in our paper.

\begin{definition}
The \emph{supremum (sup) transform} \(\Psi\) is defined by 
\[
\Psi f(x, y) := -y \sup_{w \in B(\epsilon)} (-y) f(x+w). 
\]
Additionally, we define \(\Psi \functions\) to be the transformed function class
\[
\Psi \functions
:=
\{\Psi f: f \in \functions\}.
\]
\label{defSupTransform}
\end{definition}

%For brevity, we refer to \(\Psi f\) as the ``sup transform.''
We now have the following result, proved in Section~\ref{SecSupTree}:
\begin{proposition}
Let \(\ell(f, z)\) be a loss function that is monotonically decreasing in \(yf(x)\). Then
\[
\sup_{w \in B(\epsilon)} \ell(f, z+w)
=
\ell(\Psi f, z).
\]
\label{propSupTransformEquality}
\end{proposition}

\begin{remark}
The consequence of the supremum transformation can be seen by taking the expectation:
\[
\expect_{\dist} \left[ \sup_{w \in B(\epsilon)} \ell(f, z)\right]
=
\expect_{\dist} \left[\ell(\Psi f, z)\right].
\]
Thus, we can bound the adversarial risk of a function \(f\) with a bound on the usual risk of \(\Psi f\) via Rademacher complexities.
For linear classifiers, we shall see momentarily that the supremum transformation can be calculated exactly.
\end{remark}

%---------------------------------------------%
%---------------------------------------------%
\subsection{The supremum transformation and linear classification}
\label{secLinear}

We start with an explicit formula for the supremum transform.

\begin{proposition}
Let \(f(x) = \theta^{\trans} x + b\).
Then the supremum transformation takes the explicit form 
\[
\Psi f(x, y)
=
\theta^{\trans} x + b - y \epsilon \|\theta \|_{q},
\]
where $q$ satisfies $\frac{1}{p} + \frac{1}{q} = 1$.
\label{propSupLinear}
\end{proposition}
The proof is contained in Section~\ref{SecSupTree}.

Next, the key ingredient to a generalization bound is an upper bound on the Rademacher complexity of $\Psi \functions$.

\begin{lemma}
Let \(\functions_{\linear}\) be a compact linear function class such that
\(\|\theta\|_{2} \leq M_{2}\) and
\(\|\theta\|_{q} \leq M_{q}\) 
for all \(f \in \functions_{\linear}\), where \(f(x) = \theta^{\trans} x + b\). Suppose $\|x_i\|_2 \le R$ for all $i$.
Then we have
\[
\emprad_{n}(\Psi \functions_{\linear})
\leq 
\frac{M_{2} R}{\sqrt{n}} 
+
\frac{\epsilon M_{q}}{2 \sqrt{n}}.
\]
\label{lemmaLinearERC}
\end{lemma}
The proof is contained in Section~\ref{subsecRademacher}.

This leads to the following upper bound on adversarial risk, proved in Section~\ref{subsecCorProofs}:

\begin{corollary}
Let \(\functions_{\linear}\) be a collection of linear classifiers such that, for any classifier \(f(x) = \theta^{\trans}x + b\) in \(\functions_{\linear}\),
we have \(\|\theta\|_{2} \leq M_{2}\) and \(\|\theta\|_{q} \leq M_{q}\).
Let \(R\) be a constant such that \(\|x_{i}\|_{2} \leq R\) for all \(i\).
Let \(\ell\) be a 1-Lipschitz loss bounded by \(1\).
Then for any \(f \in \functions_{\linear}\), we have
\begin{align}
& \begin{aligned}
\label{EqnLin1}
\risk_{\robust}(\ell, f) = \expect_{\dist} \left[\ell(\Psi f, z)\right]
&\leq 
\frac{1}{n} \sum_{i = 1}^{n}  \ell(\Psi f, z_{i})
+
2 \frac{M_{2} R}{\sqrt{n}}
+
\frac{\epsilon M_{q}}{\sqrt{n}}
+
3\sqrt{\frac{\log \frac{2}{\delta}}{2n}}
\end{aligned}
\end{align}
and
\begin{align}
& \begin{aligned}
\label{EqnLin2}
\risk_{\robust}(\lhinge, f)
&\leq 
\frac{1}{n} \sum_{i = 1}^{n}  \lhinge(f, z_{i})
+
\epsilon \|\theta\|_{q}
\frac{1}{n} \sum_{i = 1}^{n} 
\lhzo(\Psi f, z_{i}) 
+
2 \frac{M_{2} R}{\sqrt{n}}
+
\frac{\epsilon M_{q}}{\sqrt{n}}
+
3\sqrt{\frac{\log \frac{2}{\delta}}{2n}},
\end{aligned}
\end{align}
with probability at least $1-\delta$.
\label{corGeneralizationLinear}
\end{corollary}

Note that if \(\ell\) upper-bounds the indicator loss, then the generalization bound of Corollary~\ref{corGeneralizationLinear} also bounds the risk of the indicator loss. In particular, the upper bound~\eqref{EqnLin2} is an upper bound on the adversarial risk for the 01-loss.

\begin{remark}
An immediate question is how our adversarial risk bounds compare with the case when perturbations are absent. Plugging \(\epsilon = 0\) into the equations above yields the usual generalization bounds of the form
\begin{equation*}
\expect_P \left[\ell(f,z)\right]
\le
\frac{1}{n} \sum_{i=1}^n \ell(f,z_i) + \frac{C_1}{\sqrt{n}},
\end{equation*}
so the effect of an adversarial perturbation is essentially to introduce an additional $\mathcal{O}\left(n^{-1/2}\right)$ term, as well as an additional contribution to the empirical risk that depends linearly on $\epsilon$. The additional empirical risk term vanishes if \(f\) classifies adversarially perturbed points \(z + w\) correctly, since $\ell_{h, 01}(\Psi f, z) = 0$ in that case.
%unlike in the case of Wasserstein distributionally robust optimization.
%\textcolor{red}{I am still not fully convinced that the extra $\epsilon \|\theta\|_q$ term is of the right order. Is there some further explanation you can add?}
\end{remark}

\begin{remark}
\label{remarkLinearEmpBound}
%As in the proof of Lemma~\ref{lemmaLinearEmpiricalBound} is not convex due to the presence of the \(01\)-loss function, but we can easily provide a further crude upper bound (cf.\ equations~\eqref{EqnFirst} and~\eqref{EqnSecond}):
%\begin{equation}
%\frac{1}{n} \sum_{i = 1}^{n} 
%\lhinge(\Psi f, z_{i} + w_{i}) - \lhinge(f, z_{i})
%\leq 
%\epsilon \|\theta\|_{q}  
%\label{eqnConvexUpper}
%\end{equation}
%\textcolor{red}{Maybe a bit more explanation here. I don't understand why the following sentences are true.} 
Clearly, we could further upper-bound the regularization term in equation~\eqref{EqnLin2} by $\epsilon \|\theta\|_q$. This is essentially the bound obtained for the empirical risk for Wasserstein distributionally robust linear classification \citep{gao2017}. However, this bound is loose when a good robust linear classifier exists, i.e., when $\sum_{i = 1}^{n}  \lhzo(\Psi f, z_{i})$ is small relative to \(n\). Thus, when good robust classifiers exist, distributional robustness is relatively conservative for solving the adversarially robust problem (cf.\ Appendix~\ref{appAdvDist}).
\end{remark}

%---------------------------------------------%
%---------------------------------------------%
\subsection{The tree transformation and neural networks}
\label{subsecNeuralNetworks}

In this section, we consider adversarial risk bounds for neural networks. We begin by introducing the tree transformation, which unravels the neural network into a tree in some sense.

\begin{definition}
Let \(f\) be a neural network given by
\[
f(x) 
=
A^{(d + 1)} s_{d}\left( A^{(d)} s_{d - 1}\left( A^{(d - 1)}\ldots s_{1}\left(A^{(1)} x\right)\right)\right).
\]
Define the terms \(w^{(j_{2: d + 1})}_{f}\) and \(\sgn(f, j_{2 : d + 1})\) by
\begin{equation}
\label{EqnThreeStar}
w^{(j_{1 : d + 1})}_{f} 
:= 
-y \sgn(f, j_{2 : d + 1})
\epsilon \left\|a^{(1)}_{j_{1}}\right\|_{q} 
\end{equation}
and
\[
\sgn(f, j_{1 : d + 1}) 
:=
\sgn\left(\prod_{k = 1}^{d + 1} a^{(k + 1)}_{j_{k + 1}, j_{k}}
\right).
\]
The \emph{tree transform} \(Tf\) is defined by 
\begin{equation}
\label{EqnTwoStar}
Tf(x, y) 
:=
\sum_{j_{d} = 1}^{J_{d}}
a^{(d + 1)}_{1, j_{d}} 
s_{d}\left( \sum_{j_{d - 1} = 1}^{J_{d - 1}} a^{(d)}_{j_{d}, j_{d - 1}} 
%\right. \\ &\qquad \left. 
s_{d - 1}\left(  \ldots
\sum_{j_1 = 1}^{J_1} a_{j_2, j_1}^{(2)} s_{1}\left(\left(a^{(1)}_{j_{1}}\right)^{\trans} x 
+
w^{(j_{1: d + 1})}_{f}
\right)
\right)
\right).
\end{equation}
\label{defTreeTransform}
\end{definition}

Intuitively, the tree transform~\eqref{EqnTwoStar} can be thought of as a new neural network classifier where the adversary can select a different worst-case perturbation \(w_f\) for each path through the neural network from the input to the output indexed by $(j_2, \dots, j_{d+1})$.
This leads to \(\prod_{k = 2}^{d + 1} J_{k}\) distinct paths through the network for given inputs \(x\) and \(w_{f}^{(j_{2 : d + 1})}\), and if these paths were laid out, they would form a tree (see Section~\ref{AppTree}).

Next, we show that the risk of the tree transform upper-bounds the adversarial risk of the original neural network. The proof is contained in Section~\ref{SecSupTree}.

\begin{proposition}
Let \(\ell(f, z)\) be monotonically decreasing in \(y f(x)\).
%Let \(f\) be a neural network
%where the activation functions are monotonically increasing.
Then we have the inequality
\[
\sup _{w \in B(\epsilon)} \ell(f, z + w)
=
\ell(\Psi f, z)
\leq 
\ell(Tf, z).
\]
\label{propTreeLoss}
\end{proposition}

As an immediate corollary, we obtain
\[
\expect \left[\sup _{w \in B(\epsilon)} \ell(f, z + w)\right]
\leq 
\expect \ell(Tf, z),
\]
so it suffices to bound this latter expectation. We have the following bound on the Rademacher complexity of \(T \functions_{\nets}\), proved in Section~\ref{subsecRademacher}:

\begin{lemma}
Let \(\functions_{\nets}\) be a class of neural networks of depth \(d\) satisfying \(\|A_{j}\|_{\infty} \leq \alpha_{j}\) and \(\|A_{j}\|_{F} \leq \alpha_{1, F}\), for each \(j = 1, \ldots, d+1\), and let \(\alpha = \prod_{j = 1}^{d + 1} \alpha_{j}\). 
%\textcolor{red}{should this be $\alpha_1$?}
Additionally, suppose \(\max_{j = 1, \ldots, J_{1}}\|a_{j}^{(1)}\|_{q} \leq \alpha_{1, q}\) and \(\|x_{i}\|_2 \le R\) for all $i$.
Then we have the bound
\begin{align*}
& \begin{aligned}
\emprad_{n}(T \functions_{\nets})
&\leq 
\alpha\left( \frac{\alpha_{1, F}}{\alpha_{1}}R + \frac{\alpha_{1, q}}{\alpha_{1}} \epsilon\right) 
\cdot \frac{\sqrt{2d \log 2} + 1}{\sqrt{n}}.
   \end{aligned}
\end{align*}
\label{lemmaNNRad2}
\end{lemma}

Finally, we have our adversarial risk bounds for neural networks. The proof is contained in Section~\ref{subsecCorProofs}.

\begin{corollary}
Let \(\functions_{\nets}\) be a class of neural networks of depth \(d\). 
Let \(g_{i}(a) = \ellbar_{\text{xe}}(\delta(a), y_{i})\).
Let \(\ell\) be a 1-Lipschitz loss bounded by 1.
Under the same assumptions as Lemma~\ref{lemmaNNRad2}, for any $f \in \functions_{\nets}$, we have
%where each \(\|A_{j}\|_{\infty} \leq \alpha _{j}\) for each \(j\),
%and define \(\alpha = \prod_{j = 1}^{d + 1} \alpha_{j}\).
%Additionally, let \(\|A_{j}\|_{F} \leq \alpha_{1, F}\) and \(\max_{j = 1, \ldots, J_{1}}\|a_{j}^{(1)}\|_{q} \leq \alpha_{1, q}\).
%Let \(R\) be a bound on the \(\|x_{i}\|\).
the upper bounds
\begin{align*}
&\begin{aligned}
\risk_{\robust}(\ell, f) = \expect_{P} \left[\ell (Tf, z)\right]
&\leq 
\frac{1}{n} \sum_{i = 1}^{n} \ell(Tf, z_{i})  +
3\sqrt{\frac{\log \frac{2}{\delta}}{2n}} 
%\\ &\qquad 
+ 
2 \alpha\left( \frac{\alpha_{1, F}}{\alpha_{1}}R + \frac{\alpha_{1, q}}{\alpha_{1}} \epsilon\right) 
%\\ &\qquad  \times 
\frac{\sqrt{2d \log 2} + 1}{\sqrt{n}}
\end{aligned}
\end{align*}
and
\begin{align}
\label{EqnGenNN}
&\begin{aligned}
\risk_{\robust}(\lxe, f)
&\leq 
\frac{1}{n} \sum_{i = 1}^{n} \lxe(f, z_{i}) 
+
3\sqrt{\frac{\log \frac{2}{\delta}}{2n}}
% \\ &\qquad 
+
\epsilon 
\max_{j = 1, \ldots, J_{1}} \left\|a^{(1)}_{j}\right\|_{q} 
\prod_{j = 2}^{d + 1} \|A_{j}\|_{\infty}
%\\ &\qquad \times 
\frac{1}{n} \sum_{i = 1}^{n} |g'_{i}(Tf(x_{i}, y_i))|
\\ &\qquad 
+ 
2 \alpha\left( \frac{\alpha_{1, F}}{\alpha_{1}}R + \frac{\alpha_{1, q}}{\alpha_{1}} \epsilon\right) 
%\\ &\qquad \times 
\frac{\sqrt{2d \log 2} + 1}{\sqrt{n}},
\end{aligned}
\end{align}
with probability at least $1-\delta$.
\label{corNNGeneralization}
\end{corollary}

\begin{remark}
As in the linear case, we can essentially recover preexisting non-adversarial risk bounds \citep{bartlett2017, golowich2018} by setting \(\epsilon = 0\).
Again, the effect of adversarial perturbations on the adversarial risk is the addition of $\mathcal{O}\left(n^{-1/2}\right)$ on top of the empirical risk bounds for the unperturbed loss. 
Finally, the bound~\eqref{EqnGenNN} includes an extra perturbation term that is linear in $\epsilon$, with coefficient reflecting the Lipschitz coefficient of the neural network, as well as a term $\frac{1}{n} \sum_{i=1}^n |g_i'(Tf(x_i, y_i))|$, which decreases as \(Tf\) improves as a classifier since \(|g'_{i}(Tf(x_{i}, y_i))|\) is small when \(\lxe(Tf, z_{i})\) is small. A similar term appears in the bound~\eqref{EqnLin2}.
\end{remark}

%---------------------------------------------%
%---------------------------------------------%
\subsection{A visualization of the tree transform}
\label{AppTree}

In this section, we provide a few pictures to illustrate the tree transform.
Consider the following two-layer network with two hidden units per layer:
\[
f(x) = A^{(3)} s_{2}(A^{(2)} s_{1}(A^{(1)} x)).
\]
We begin by visualizing 
\(\sup_{w \in B(\epsilon)} f(x + w)\) in Figure~\ref{figNetwork}.

\begin{figure}[!h]
\centering 
\captionsetup{width=0.8\linewidth}
\includegraphics[scale = 0.14]{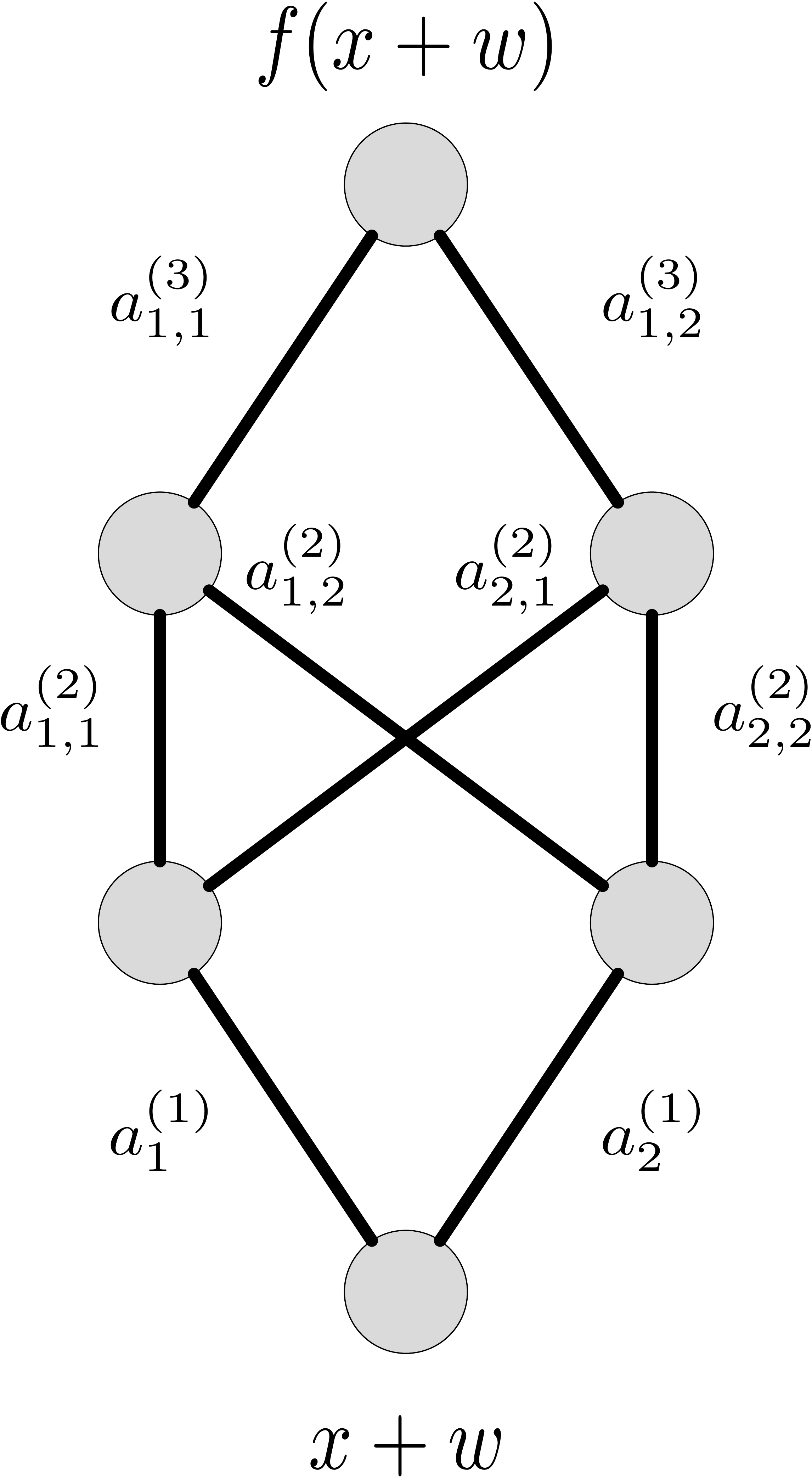}
\caption{A visualization of \(f(x + w)\). The input \(x + w\) is fed up through the network.}
\label{figNetwork}
\end{figure}

Next, we examine what happens when the supremum is taken inside the first layer. The resulting transformed function (cf.\ Lemma~\ref{lemmaTreeIterate} in Section~\ref{secProofs}) becomes
\begin{align}
&\begin{aligned}
g(x, y)
&=
\sum_{j_{3} = 1}^{2} a^{(3)}_{1, j_{3}}
s_{2}\left( 
\sgn\left(-ya^{(3)}_{1, j_{3}}\right)
\sup_{w^{(j_{3})} \in B(\epsilon)}
%\right. \\ &\; \; \left.
\sgn\left(-ya^{(3)}_{1, j_{3}}\right)
A^{(2)} s_{1}\left(A^{(1)}\left(x + w^{(j_{3})}\right) \right)
\right).
\label{eqnTreeStep1}
\end{aligned}
\end{align}
The corresponding network is shown in Figure~\ref{figTreeStep1}.

\begin{figure}[!h]
\centering 
\captionsetup{width=0.8\linewidth}
\includegraphics[scale = 0.14]{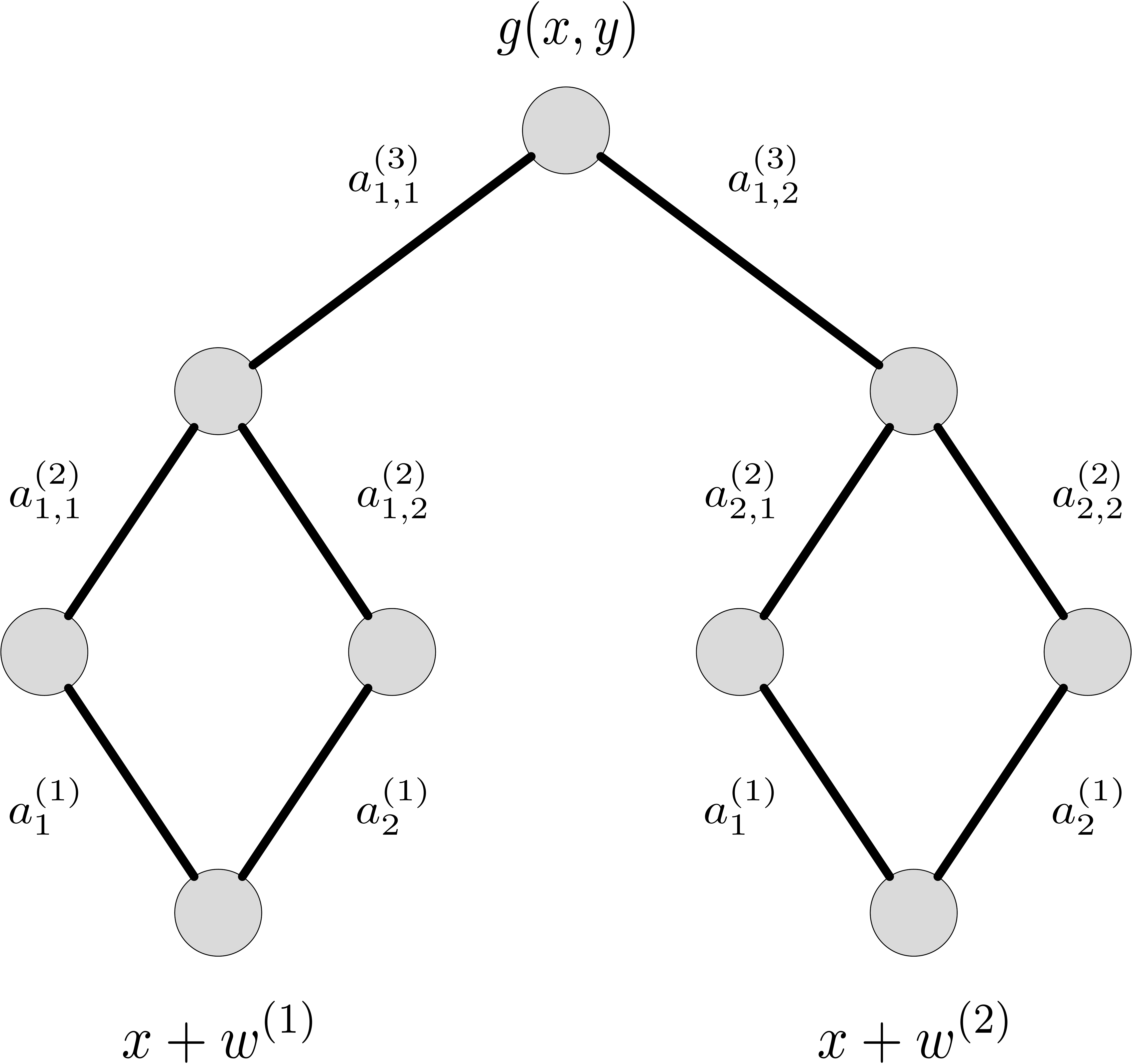}
\caption{A visualization of the function \(g(x, y)\) of equation~\eqref{eqnTreeStep1}. Note that  two different perturbations, \(w^{(1)}\) and \(w^{(2)}\), are fed upward through different paths in the network.}
\label{figTreeStep1}
\end{figure}

Finally, we examine the entire tree transform.
This is
\begin{align}
& \begin{aligned}
Tf(x, y) 
&=
\sum_{j_{3} = 1}^{2} a^{(3)}_{1, j_{3}}
s_{2}\left( 
\sum_{j_{2} = 1}^{J_{2}} a^{(2)}_{j_{3}, j_{2}}
% \right. \\ & \qquad \left. 
 s_{1}\left( 
\sgn\left(-y a^{(3)}_{1, j_{3}} a^{(2)}_{j_{3}, j_{2}}\right)
%\right. \right. \\ & \qquad \left. \left.
\sup_{w^{(j_{2}, j_{3})}}
\left(a^{(1)}_{j_{2}}\right)^{\trans}
\left(x + w^{(j_{2}, j_{3})}\right) \right)
\right).
\label{eqnTreeStep2}
   \end{aligned}
\end{align}
%We have not simplified entirely here, but this essentially gives the tree transform in that we do not have to pass to lower layers via any further approximation.
The result is shown in Figure~\ref{figTreeStep2}. In particular, the visualization of the network resembles a tree, which is the reason we call \(T\) the tree transform. 

\begin{figure}[!h]
\centering 
\captionsetup{width=0.8\linewidth}
\includegraphics[scale = 0.14]{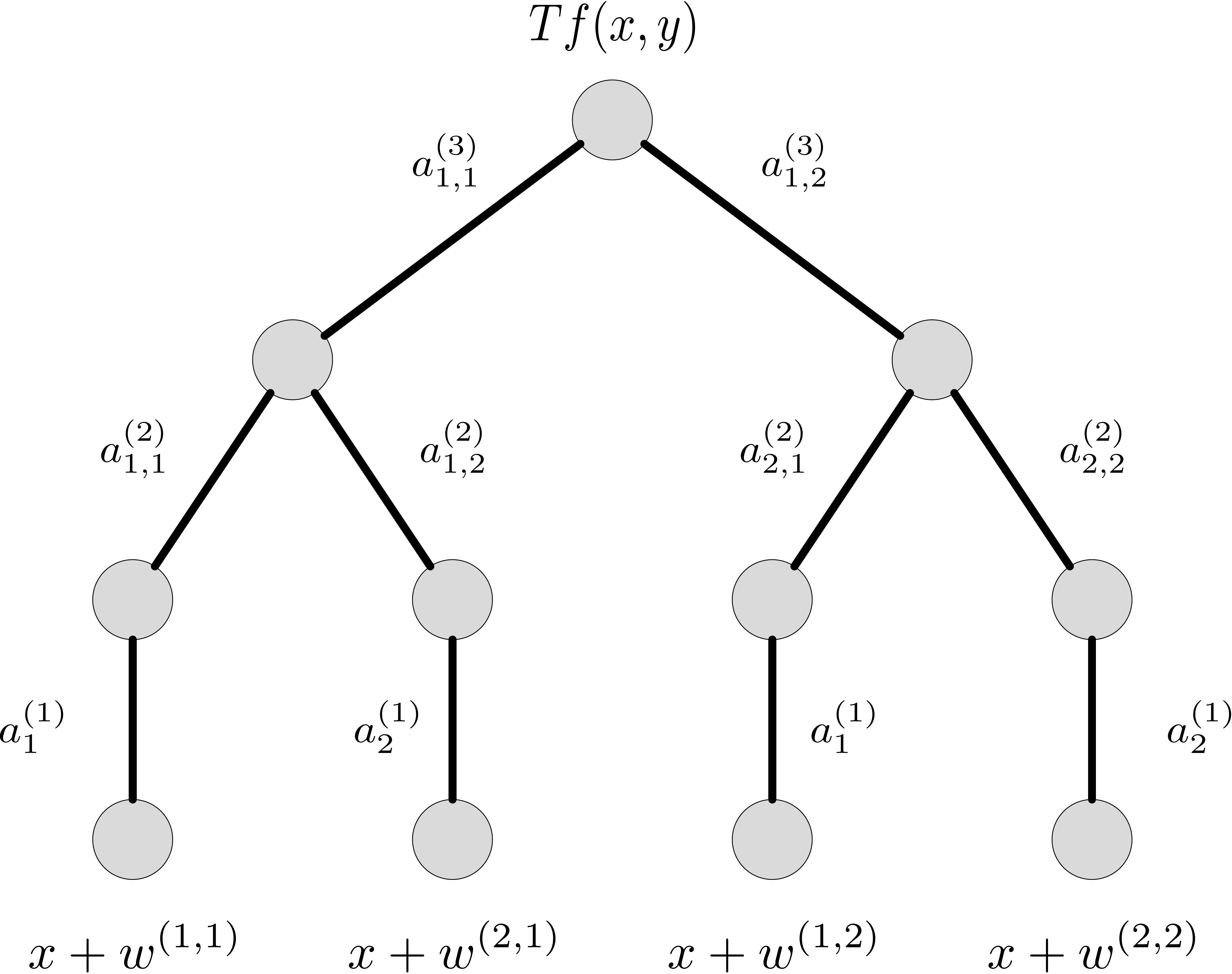}
\caption{A visualization of the function \(Tf(x, y)\) in equation~\eqref{eqnTreeStep2}. Note that  four distinct perturbed inputs are fed through the network via different paths. The resulting tree-structured graph leads to the name ``tree transform.''}
\label{figTreeStep2}
\end{figure}

%---------------------------------------------%
%---------------------------------------------%
\section{Extension to multiclass classification}
\label{secMulticlass}

In this section, we discuss how our results for binary classification may be extended to multiclass classification.
Before stating the results, we first discuss changes to our problem setup. Proofs of all results are contained in Appendix~\ref{AppMulticlass}.

\subsection{Setup}
\label{subsecMulticlassSetup}

In the case of multiclass classification, we need to make small adjustments to the label space \(\yspace\), the loss function \(\ell\), and the function classes \(\functions_{\linear}\) and \(\functions_{\nets}\).
First, our new space of labels is
\[
\yspace 
= 
\left\{y \in \{+1, -1\}^{K}: \text{exactly one index } y^{(i)} = +1 \right\},
\]
where the number of labels is denoted by \(K\). We write the components of \(y \in \yspace\) as 
\(y = \left(y^{(1)}, \ldots, y^{(K)}\right)\), reserving subscripts for distinct data points.
We also define \(\tilde{y} \in \{1, \ldots, K\}\) to be the index of the entry \(+1\). 

Next, we discuss the loss. We proceed via the method described in \cite{mohri2012}, which derives generalization error bounds in terms of the Rademacher complexity of a modified function class. Our multiclass results are specific to the multiclass margin loss, which is essentially a truncated hinge loss. The multiclass margin is defined by
\[
m_f(z) := f_{\tilde{y}}(x)
-
\max_{i \neq \tilde{y}} f_{i}(x).
\]
Next, define the function 
\[
\phi_{\rho}(t) 
\begin{cases}
0 & \rho \leq t \\ 
1 - \frac{t}{\rho} & 0 \leq t \leq \rho \\ 
1 & t \leq 0,
\end{cases}
\]
where \(\rho > 0\) is some constant. The margin loss \(\lmargin\) is defined as \(\lmargin(f(x), y) = \phi_{\rho}(m_f(z))\). The key observation is that $\lmargin$ is indeed a valid surrogate loss for the indicator loss, since
\begin{equation*}
\ind\left\{f_{\tilde{y}}(x) \neq \max_{i \neq \tilde{y}} f_i(x)\right\} \le \lmargin(f(x), y).
\end{equation*}
Note that by rescaling the hinge loss in binary classification, we could have introduced the margin parameter \(\rho\) previously.
However, we omitted this extra notation to ease readability.

Finally, we modify the function classes \(\functions_{\linear}\) and \(\functions_{\nets}\).
For the former, we consider linear functions \(f\) of the form
\[
f(x) = \Theta x + b,
\]
where \(\Theta\) is a \(K \times m\) matrix and \(b\) is in \(\reals^{K}\). We index the rows of $\Theta$ as follows: \(\Theta = (\theta_{1}, \ldots, \theta_{K})^{\trans}\).

A neural network \(f\) again has the form 
\begin{equation}
\label{EqnNNDef}
f(x) = A^{(d + 1)} s_{d}(A^{(d)} s_{d - 1}(\ldots s_{1}(A^{(1)} x))),
\end{equation}
but in this case, \(A^{(d + 1)}\) is a \(K \times J_{d}\) matrix instead of a \(1 \times J_{d}\) matrix.

%---------------------------------------------%
%---------------------------------------------%
\subsection{Multiclass results}
\label{subsecMulticlassResults}

As in the case of binary classification, we need some notion of monotonicity of the loss function, where we now have a function $\ell: \real^K \times \real^K \rightarrow \real$. We begin with the following definition, which is analogous to the monotonicity property in Definition~\ref{DefMonotone}:

\begin{definition}
We say that a loss function \(\ell\) is \emph{coordinate-wise decreasing} in \(y \circ f(x)\) if \(\ell(f(x), y)\) is increasing in \(f_{k}(x)\) for \(y^{(k)} = -1\) and decreasing in \(f_{k}(x)\) for \(y^{(k)} = +1\).
\label{defCWD}
\end{definition}

It is easy to check that the margin loss is coordinate-wise decreasing. Furthermore, we now show that coordinate-wise decreasing functions satisfy a certain monotonicity property. For $y \in \yspace$ and vectors $a, b \in \real^K$, we write \(a \preceq_{y} b\) if \(a_{k} \geq b_{k}\) for \(y^{(k)} = +1\) and \(a_{k} \leq b_{k}\) for \(y^{(k)} = -1\). We then have the following lemma:

\begin{lemma}
Let \(\ell\) be a loss function that is coordinate-wise decreasing in \(y \circ f(x)\), and suppose \(a \preceq_{y} b\). Then \(\ell(a, y) \leq \ell(b, y)\).
\label{lemmaMonotonic}
\end{lemma}

%---------------------------------------------%
%---------------------------------------------%
Next, we define the required transform.
%First, define the component-wise supremum to be 
%\[
%\csup_{w \in W} g(w) 
%=
%\left(\sup _{w \in W} g_{1}(w), \ldots, \sup _{w \in W} g_{K}(w) \right).  
%\]

\begin{definition}
The \emph{multiclass supremum transform} \(\Psi\) is defined componentwise by
%\[
%\Psi f(x, y)
%=
%(-y) \circ \csup_{w \in B(\epsilon)} (-y) \circ f(x + %w).
%\]
%Equivalently, we have 
\[
(\Psi f)_{k}(x, y) = \Psi (f_{k})\left(x, y^{(k)}\right).
\]
\label{defMulticlassSup}
\end{definition}

In other words, the $k^{\text{th}}$ component of the multiclass sup transform is the binary sup transform on the componentwise classifier $f_k$ of the $k^{\text{th}}$ component.

We now derive the following critical inequality:

\begin{proposition}
For any \(w\) in \(B(\epsilon)\), we have
\begin{align*}
& \begin{aligned}
f(x + w) 
&\preceq_{y}
\Psi f(x, y).
   \end{aligned}
\end{align*}
Thus,
\[
\sup_{w \in B(\epsilon)} \ell\left(f(x + w), z\right)
\leq 
\ell\left(\Psi f(x, y), z\right).
\]
\label{propPsiMonotone}
\end{proposition}

Note that unlike in the binary classification case (Proposition~\ref{propSupTransformEquality}), we do not have exact equality in the multiclass setting.
%As with the case of the inequalities in obtaining the tree transform for neural networks, the issue is that the argument to the function to be maximized is in more than one dimension.
%---------------------------------------------%
%---------------------------------------------%
\subsubsection{Linear classification}
\label{subsecLinearMulticlass}
 
As in the case of binary classification, we can write the multiclass supremum transform of a linear classifier explicitly.

\begin{proposition}
Let \(f(x) = \Theta x + b\) be a linear classifier.
%Let \(\theta_{k}\) denote the \(k\)th row of \(\Theta\) as a column vector.
Then 
\[
\Psi f_{k}(x, y) 
= 
\theta_{k}^{\trans} x + b_{k} - y^{(k)} \epsilon \|\theta_{k}\|_{q},
\]
for all $1 \le k \le K$.
\label{propMCLinearSupEqn}
\end{proposition}

Since the definition of the multiclass supremum transform is componentwise, the proof is exactly the same as in the case of binary classification (Proposition~\ref{propSupLinear}), so we omit the proof.

Next, we obtain a generalization bound.
We use Lemma~\ref{lemmaMulticlassGeneralization} from Appendix~\ref{appAuxLemmas}.
To instantiate this bound, we need to derive a bound on the Rademacher complexity of the function class
\[
\Pi_{1}(\functions) 
:= 
\left\{
(x, y) \mapsto f_{\tilde{y}}(x): y \in \yspace, f \in \functions
\right\}.
\]
Fortunately, the proof from the case of binary classification translates almost exactly. We have the following result:

\begin{lemma}
Let \(\functions_{\linear}\) be a compact linear function class such that
\(\|\Theta\|_{2, \infty} \leq M_{2}\) and
\(\|\Theta\|_{q, \infty} \leq M_{q}\) 
for all \(f \in \functions_{\linear}\), where \(f(x) = \Theta x + b\). Suppose $\|x_i\|_2 \le R$ for all $i$.
Then we have
\[
\emprad_{n}( \Pi_{1}(\Psi \functions_{\linear}))
\leq 
\frac{M_{2} R}{\sqrt{n}} 
+
\frac{\epsilon M_{q}}{2 \sqrt{n}}.
\]
\label{lemmaMulticlassLinearERC}
\end{lemma}

Comparing Lemmas~\ref{lemmaMulticlassLinearERC} and~\ref{lemmaLinearERC}, the main difference between the multiclass and binary classification results is that the bound in Lemma~\ref{lemmaMulticlassLinearERC} is in terms of the maximum \(2\)-norms and \(q\)-norms of the rows of the matrix \(\Theta\).
Indeed, due to the special structure of linear classifiers, the proof of the Rademacher complexity bound in the multiclass case is very similar to the case of binary classification.

The generalization bound follows then directly from Lemma~\ref{lemmaMulticlassGeneralization}:
\begin{corollary}
Let \(\functions_{\linear}\) be a collection of linear classifiers such that, for any classifier \(f(x) = \Theta x + b\) in \(\functions_{\linear}\),
we have \(\|\Theta\|_{2, \infty} \leq M_{2}\) and \(\|\Theta\|_{q, \infty} \leq M_{q}\).
Let \(R\) be a constant such that \(\|x_{i}\|_{2} \leq R\) for all \(i\).
Then for any \(f \in \functions_{\linear}\), we have
\begin{align*}
& \begin{aligned}
\label{EqnLin1}
\risk_{\robust}(\lmargin, f) 
= 
\expect_{\dist} \left[\lmargin(\Psi f, z)\right]
&\leq 
\frac{1}{n} \sum_{i = 1}^{n}  \lmargin(\Psi f, z_{i})
+
\frac{8K}{\rho} \cdot \frac{M_{2} R}{\sqrt{n}}
+
\frac{4K}{\rho} \cdot  \frac{\epsilon M_{q}}{\sqrt{n}}
+
3\sqrt{\frac{\log \frac{2}{\delta}}{2n}},
\end{aligned}
\end{align*}
with probability at least $1-\delta$.
\label{theoremMulticlassLinearGeneralization}
\end{corollary}

%\textcolor{red}{Insert some discussion. Perhaps discuss dependence on $\rho$. Compare result with $K=2$ case.}
Note that the difference between our linear binary classification result and this result is entirely due to the difference between the generalization bounds obtained in binary classification versus those obtained by multiclass classifcation. The same situation arises in our analysis of neural networks.

\begin{remark}
It is instructive to compare Corollaries~\ref{corGeneralizationLinear} and \ref{theoremMulticlassLinearGeneralization}. The main difference is that in Corollary~\ref{theoremMulticlassLinearGeneralization}, the constant multiplying the Rademacher complexity term is \(8K / \rho\) instead of \(2\). Note that the factor \(1 / \rho\) originates from the Lipschitz constant of the loss \(\lmargin\). Choosing larger values of \(\rho\) decreases the effect of the Rademacher complexity term, at the expense of possibly making the empirical risk of \(\Psi f\) with respect to $\ell_\rho$ larger.
Secondly, due to additional complexities caused by the vectorization of the margin loss, we incur a factor of \(8K\) instead of \(2\) in removing the loss.
However, this is not ideal when \(K\) is large. The linear dependence on $K$ might 
possibly be removed by alternatively considering a covering number-based bound on Rademacher complexity, although with existing tools, this would lead to additional terms depending on the depth of the network.
\end{remark}

%---------------------------------------------%
%---------------------------------------------%
\subsubsection{Neural networks}
\label{subsecNNMulticlass}

We now define the multiclass tree transform for neural networks. Let \(f\) be a neural network given by equation~\eqref{EqnNNDef}. We write the $k^{\text{th}}$ entry of \(f(x)\) as the neural network
\[
f_{k}(x) 
=
A^{(d + 1)}_{k} s_{d} \left( A^{(d)} s_{d - 1} \left(\ldots s_{1}\left( A^{(1)} x \right) \right) \right).
\]

\begin{definition}
The \emph{multiclass tree transform} is defined componentwise by
\[
(Tf)_{k}(x, y)
= 
T(f_{k})(x, y).
\]
%Define the terms \(w_{f}^{(j_{2: d + 1}, k)}\) and \(\sgn(f, j_{2: d + 1}, k)\) by
%\[
%w_{f}^{(j_{2: d + 1}, k)}
%=
%-y^{(k)} \sgn(f, j_{2: d + 1}, k) \epsilon \|a_{j_{2}}^{(1)}\|_{q}
%\]
%and
%\[
%\sgn(f, j_{2: d + 1}, k)
%=
%\sgn\left(a^{(d + 1)}_{k, j_{d + 1}}\prod_{r = 2}^{d} a^{(r)}_{j_{r + 1}, j_{r}}\right).
%\]
%Then, the tree transform \(Tf\) is defined by
%\begin{align*}
%& \begin{aligned}
%Tf_{k}(x, y) 
%&:=
%\sum_{j_{d + 1} = 1}^{J_{d + 1}} a_{k, j_{d + 1}}^{(d + 1)}
%s_{d}\left(
%\sum_{j_{d} = 1}^{J_{d}} a^{(d)}_{j_{d + 1}, j_{d}} 
%s_{d - 1}\left( \ldots
%s_{1}\left(
%\left(a^{(1)}_{j_{2}}\right)^{\trans} x + w_{f}^{(j_{2 : d + 1}, k)}
%\right)
%\right)
%\right)
%   \end{aligned}
%\end{align*}
\label{defMCTree}
\end{definition}

Since $\Psi$ and $T$ are defined componentwise, and $-y \Psi(f_k)(x,y) \le -y T(f_k)(x,y)$ for all $k$, we clearly have $\Psi f(x,y) \le_y Tf(x,y)$. Hence, \(\ell(\Psi f, z) \leq \ell(Tf, z)\) for an \(\ell\) that is coordinate-wise decreasing in \(y \circ f(x)\), by Lemma~\ref{lemmaMonotonic}. By Proposition~\ref{propPsiMonotone}, we then have
\begin{equation*}
\sup_{w \in B(\epsilon)} \ell(f(x+w),z) \le \ell(Tf(x,y),z).
\end{equation*}

%\begin{proposition}
%Let \(f\) be a neural network. Then \(\Psi f(x, y) \preceq_{y} Tf(x, y)\), so \(\ell(\Psi f, z) \leq \ell(Tf, z)\) for an \(\ell\) that is coordinate-wise decreasing in \(y \circ f(x)\).
%\label{propStuff}
%\end{proposition}

%Again, since the definitions of \(\Psi f\) and \(Tf\) are component-wise, the result in the binary case is sufficient for the multiclass case.

Accordingly, it suffices to derive a bound on the Rademacher complexity of $\Pi_1(T \functions)$: 

\begin{lemma}
Let \(\functions_{\nets}\) be a class of neural networks of depth \(d\) satisfying \(\|A_{j}\|_{\infty} \leq \alpha_{j}\) and \(\|A_{j}\|_{F} \leq \alpha_{1, F}\), for each \(j = 1, \ldots, d+1\), and let \(\alpha = \prod_{j = 1}^{d + 1} \alpha_{j}\). 
%\textcolor{red}{should this be $\alpha_1$?}
Additionally, suppose \(\max_{j = 1, \ldots, J_{1}}\|a_{j}^{(1)}\|_{q} \leq \alpha_{1, q}\) and \(\|x_{i}\|_2 \le R\) for all $i$.
Then we have the bound
\begin{align*}
& \begin{aligned}
\emprad_{n}(\Pi_{1}(T \functions))
&\leq 
\alpha\left( \frac{\alpha_{1, F}}{\alpha_{1}}R + \frac{\alpha_{1, q}}{\alpha_{1}} \epsilon\right) 
\cdot \frac{\sqrt{2d \log 2} + 1}{\sqrt{n}}.
   \end{aligned}
\end{align*}
\label{lemmaMulticlassTreeRademacher}
\end{lemma}

%Fortunately, the proof of the binary classification result translates almost exactly to the multiclass setting, and so we omit it.
The proof is nearly identical to the proof of Lemma~\ref{lemmaNNRad2} for the binary classification case, except for the addition of an extra sup over $k$, so we omit it. Finally, applying Lemma~\ref{lemmaMulticlassGeneralization}, we obtain our adversarial risk bound:
\begin{corollary}
Let \(\functions_{\nets}\) be a class of neural networks of depth \(d\). 
%Let \(g_{i}(a) = \ellbar_{\text{xe}}(\delta(a), y_{i})\).
Under the same assumptions as in Lemma~\ref{lemmaMulticlassTreeRademacher}, for any $f \in \functions_{\nets}$, we have
%where each \(\|A_{j}\|_{\infty} \leq \alpha _{j}\) for each \(j\),
%and define \(\alpha = \prod_{j = 1}^{d + 1} \alpha_{j}\).
%Additionally, let \(\|A_{j}\|_{F} \leq \alpha_{1, F}\) and \(\max_{j = 1, \ldots, J_{1}}\|a_{j}^{(1)}\|_{q} \leq \alpha_{1, q}\).
%Let \(R\) be a bound on the \(\|x_{i}\|\).
the upper bound
\begin{align*}
&\begin{aligned}
\risk_{\robust}(\lmargin, f) 
= 
\expect_{P} \lmargin (Tf, z) 
&\leq 
\frac{1}{n} \sum_{i = 1}^{n} \lmargin(Tf, z_{i})  +
3\sqrt{\frac{\log \frac{2}{\delta}}{2n}} 
%\\ &\qquad 
+ 
\frac{8K}{\rho} \alpha\left( \frac{\alpha_{1, F}}{\alpha_{1}}R + \frac{\alpha_{1, q}}{\alpha_{1}} \epsilon\right) 
%\\ &\qquad  \times 
\frac{\sqrt{2d \log 2} + 1}{\sqrt{n}},
\end{aligned}
\end{align*}
with probability at least $1-\delta$.
\label{corMulticlassTreeGeneralization}
\end{corollary}

%\textcolor{red}{Compare result to $K=2$ case.}

%---------------------------------------------%
%---------------------------------------------%
\section{Extension to regression}
\label{secRegression}

In this section, we present a further extension of our theory to the case of regression. 
We start by discussing modifications to the setup, and then we present our main results. The main idea is to define an appropriate monotonic, Lipschitz loss. Proofs of the results are contained in Appendix~\ref{AppRegression}.

\subsection{Setup}
\label{subsecRegressionSetup}

As in the case of multiclass classification, we need to make appropriate adjustments to \(\yspace\), the loss function, and our function classes.
For regression, we take \(\yspace = \reals\).
As before, let \(\zspace = \xspace \times \yspace\), and suppose we have \(n\) observations \(z_{1} = (x_{1}, y_{1}), \ldots, z_{n} = (x_{n}, y_{n})\), drawn i.i.d.\ according to some unknown distribution \(\dist\).
In the case of regression, we require the \(y_{i}\)'s to be integrable, 
i.e., to have finite expectation.
%to avoid excessively technical difficulties. \textcolor{red}{What does the last sentence mean?}

The key difference in regression is that since \(y\) is real-valued, we need to use a different loss function.
In particular, we consider the loss function 
\begin{equation}
\label{EqnDefLoss}
\lr(f, z) 
=
|f(x) - y|^{r}.
\end{equation}
The most common example is the squared error loss, which corresponds to the choice \(r = 2\).  However, our theory will apply to any $r > 0$. 
%\textcolor{red}{Does the function need to be convex, or just monotonic? What about a function that is monotonic, but not a power of the absolute value?}
%In our results, we assume that the absolute difference \(|f(x) - y|\) is bounded by some constant \(B > 0\).
%Since we generally assume that our function classes are bounded, we could achieve this through assuming that \(y\) is bounded as well.
%However, to simplify matters, we shall simply truncate the loss, giving the loss function
For technical reasons, we also need the loss to be bounded. Thus, we define the truncated loss
\[
\lrtrunc(f, z)
=
\min\{\lr(f, z), B^{r}\},
\]
for some constant $B > 0$.

Since the loss function~\eqref{EqnDefLoss} is non-monotone, we decompose the loss into two monotone components, which we will analyze separately.
Accordingly, we define the functions 
\((x)_{+} = \max\{0, x\}\) and \((x)_{-} = \max\{0, -x\}\), and define the losses
\begin{align*}
&\begin{aligned}
\lrplus(f, z) &:= \min\{(f(x) - y)_{+}^{r}, B^{r}\}, \\
\lrminus(f, z) &:= \min\{(f(x) - y)_{-}^{r}, B^{r}\}.
\end{aligned}
\end{align*}
Note that
\begin{equation}
\lrtrunc(f, z)
=
\max\left\{ 
\lrplus(f, z), 
\lrminus(f, z)
\right\}.
\label{eqnLRTRUNC}
\end{equation}
In general, we could replace the loss~\eqref{EqnDefLoss} by any bounded loss function that can be decomposed into a maximum of two monotonic, Lipschitz loss functions, such as the Huber loss.

Finally, unlike in the case of multiclass classification, our function classes \(\functions_{\linear}\) and \(\functions_{\nets}\) require no modifications from the setting of binary classification. The reason is that we needed a vector-valued output for multiclass classification, but not for binary classification or regression.
%---------------------------------------------%
%---------------------------------------------%
\subsection{Regression results}
\label{subsecRegressionResults}

As mentioned in the setup, the main difference between the regression and classification settings is the lack of monotonicity. However, due to the relation~\eqref{eqnLRTRUNC}, the cost is only a constant factor in the Rademacher complexity term.

First, we define two transforms:
\begin{definition}
The \emph{positive transform} \(\ptrans\) and \emph{negative transform} \(\ntrans\) are defined by
\begin{align*}
& \begin{aligned}
\ptrans f(x)
&:= \sup_{w \in B(\epsilon)} f(x + w) 
= \Psi f(x, -1) \\
\ntrans f(x) 
&:= \inf_{w \in B(\epsilon)} f(x + w)
= \Psi f(x, +1).
   \end{aligned}
\end{align*}
\label{defPtransNtrans}
\end{definition}

Note that unlike in classification, these functions do not depend on the label \(y\).
The key is that we can again bound the adversarial loss with the transformed losses:

\begin{proposition}
Define the loss function
\begin{align*}
& \begin{aligned}
\lrtruncpn(f, g, z) 
&:= 
\max\left\{ 
\min\left\{(f(x) - y)_{+}^{r}, B^{r}\right\}, 
\min\left\{(g(x) - y)_{-}^{r}, B^{r}\right\}
\right\}.
   \end{aligned}
\end{align*}
Then we have the equality
\begin{align*}
& \begin{aligned}
\sup_{w \in B(\epsilon)} \lrtrunc(f, z + w)
&=
\lrtruncpn(\ptrans f, \ntrans f, z).
   \end{aligned}
\end{align*}
\label{propRegressionLossMinMax}
\end{proposition} 

The proof simply amounts to using equation~\eqref{eqnLRTRUNC}, rearranging, and using monotonicity.
Additionally, this form is useful for simplifying via Rademacher complexities. We also need to be a bit more careful because the function $\ell$ is no longer 1-Lipschitz.
We define the composition of the loss and a function class \(\functions\) as 
\[
\ell \circ \functions
:=
\left\{(x, y) \mapsto \ell(f(x), z): f \in \functions\right\},
\]
and similarly define the function class
\begin{equation*}
\lrtruncpn \circ (\functions, \mathcal{G}) := \left\{(x,y) \mapsto \ell_{r,B}^{\pm}(f, g, z): f \in \functions, g \in \mathcal{G}\right\}.
\end{equation*}

We now state a useful bound on the Rademacher complexity:

\begin{lemma}
Let \(\functions\) and \(\mathcal{G}\) be function classes.
Then we have the Rademacher complexity bound
\begin{align*}
&\begin{aligned}
\emprad_{n}(\lrtruncpn \circ (\functions, \mathcal{G}))
&\leq 
\emprad_{n}(\lrplus \circ \functions)
+ 
\emprad_{n}(\lrminus \circ \mathcal{G})
\leq 
rB^{r - 1} \emprad_{n}(\functions)
+
rB^{r - 1} \emprad_{n}(\mathcal{G}).
\end{aligned}
\end{align*}
\label{lemmaRademacherRegressionLoss}
\end{lemma}

In the settings we consider, we have \(\emprad_{n}(\functions) = \emprad_{n}(\mathcal{G})\).
Thus, we can think about the cost of non-monotonicity as being a factor of \(2\) in the Rademacher complexity.

%---------------------------------------------%
%---------------------------------------------%
\subsubsection{Linear regression}
\label{subsecLinearRegression}

We now consider the case of linear regression, where we can use the positive and negative transform directly.

\begin{lemma}
Let \(\functions_{\linear}\) be a class of functions \(f\) such that \(f(x) = \theta^{\trans} x + b\).
Then we have the identities
\begin{align*}
 & \begin{aligned}
 \ptrans f(x) &= \theta^{\trans}x + b + \epsilon \|\theta\|_{q}, \\
 \ntrans f(x) &= \theta^{\trans}x + b - \epsilon \|\theta\|_{q}.
    \end{aligned}
 \end{align*} 
\label{lemmaLinearPosNeg}
\end{lemma}

The proof is immediate from the definitions of \(\Psi_+ f\) and $\Psi_- f$ and Proposition~\ref{propSupLinear}, which gives the form of $\Psi f$. Our next step is to prove a bound on the Rademacher complexity.

\begin{corollary}
Let \(\functions_{\linear}\) be a class of functions \(f\) such that \(f(x) = \theta^{\trans}x + b\), 
\(\|\theta\|_{2} \leq M_{2}\),
and \(\|\theta\|_{q} \leq M_{q}\).
Suppose that \(\|x_{i}\|_{2} \leq R\) for each \(i = 1, \ldots, n\).
Then both \(\emprad_{n}(\ptrans \functions_{\linear})\) and \(\emprad_{n}(\ntrans \functions_{\linear})\) are upper-bounded by 
\[
\emprad_{n}(\functions_{\linear})
+ 
\frac{\epsilon}{n}
\expect_{\sigma} \sup_{f \in \functions_{\linear}}  
\sum_{i = 1}^{n} \sigma_{i} \|\theta\|_{q} 
\leq 
\frac{M_{2}R}{\sqrt{n}} + \frac{\epsilon M_{q}}{2\sqrt{n}}.
\]
\label{corRademacherLinearPostNeg}
\end{corollary}

The proof leverages the bound on Rademacher complexities derived for binary classification in Lemma~\ref{lemmaLinearERC}. This leads to the following risk bound:

\begin{corollary}
Let \(\functions_{\linear}\) be as in Corollary~\ref{corRademacherLinearPostNeg}.
We have the risk bound
\begin{align*}
& \begin{aligned}
\risk_{\robust}(\lrtrunc, f)
&=
\expect \left[\lrtruncpn(\ptrans f, \ntrans f, z)\right]
\leq 
\frac{1}{n} \sum_{i = 1}^{n} \lrtruncpn(f, z_{i})
+ 
4rB^{r - 1} \left(\frac{M_{2}R}{\sqrt{n}} + \frac{\epsilon M_{q}}{2\sqrt{n}}\right)
+
3 B^{r}\sqrt{\frac{\log\frac{2}{\delta}}{2n}},
   \end{aligned}
\end{align*}
with probability at least $1-\delta$.
\label{corLinRegRiskBound}
\end{corollary}

Since the Rademacher complexity bounds for \(\emprad_{n}(\ptrans \functions_{\linear})\) and \(\emprad_{n}(\ntrans \functions_{\linear})\) are the same as that for \(\emprad_{n}(\Psi \functions_{\linear})\), the differences in the generalization bound of Corollary~\ref{corLinRegRiskBound} compared with the binary classification bound (Corollary~\ref{corGeneralizationLinear}) are due entirely to the loss function, which can be seen in Lemma~\ref{lemmaRademacherRegressionLoss}.
First, because the loss is \(rB^{r - 1}\)-Lipschitz, this term appears outside the Rademacher complexity.
Second, since we decomposed the Rademacher complexity into two separate Rademacher complexities, we gained a factor of \(2\).
Thus, compared with the binary classification result, we have  an extra \(2rB^{r - 1}\) in the Rademacher complexity term.
Finally, because the use of the bounded differences inequality requires a loss bounded by \(1\), we have a factor of \(B^{r}\) in the final term.
%\textcolor{red}{Compare result with classification case.}

%---------------------------------------------%
%---------------------------------------------%
\subsubsection{Neural networks}
\label{subsecNNRegression}

For neural networks, we again need to push the supremum through the layers of the network.
Thus, we define the positive and negative tree transforms.

\begin{definition}
Let \(f\) be a neural network. The positive and negative tree transforms \(\treepos\) and \(\treeneg\) are defined by
\begin{align*}
& \begin{aligned}
\treepos f(x) &:= Tf(x, -1), \\ 
\treeneg f(x) &:= Tf(x, +1).
   \end{aligned}
\end{align*}
%given by
%\[
%f(x) 
%=
%A^{(d + 1)} s_{d}\left( A^{(d)} s_{d - 1}\left( A^{(d - 1)}\ldots s_{1}\left(A^{(1)} x\right)\right)\right).
%\]
%Define the terms \(w^{(j_{2: d + 1})}_{f}\) and \(\sgn(f, j_{2 : d + 1})\) by
%\begin{equation}
%\label{EqnThreeStar}
%w^{(j_{2: d + 1})}_{f} 
%:= 
%\sgn(f, j_{2 : d + 1})
%\epsilon \left\|a^{(1)}_{j_{2}}\right\|_{q} 
%\end{equation}
%and
%\[
%\sgn(f, j_{2 : d + 1}) 
%:=
%\sgn\left(\prod_{k = 2}^{d + 1} a^{(k)}_{j_{k + 1}, j_{k}}
%\right).
%\]
%Then the positive tree transform \(\treepos f\) is defined by 
%\begin{equation}
%\label{eqnTreePos}
%\treepos f(x) :=
%\sum_{j_{d + 1} = 1}^{J_{d + 1}}
%a^{(d + 1)}_{1, j_{d + 1}} 
%s_{d}\left( \sum_{j_{d} = 1}^{J_{d}} a^{(d)}_{j_{d + 1}, j_{d}} 
%\right. \\ &\qquad \left. 
%s_{d - 1}\left(  \ldots
%\sum_{j_2 = 1}^{J_2} a_{j_3, j_2}^{(2)} s_{1}\left(\left(a^{(1)}_{j_{2}}\right)^{\trans} x 
%+
%w^{(j_{2: d + 1})}_{f}
%\right)
%\right)
%\right).
%\end{equation}
%Similarly, the negative tree transform \(\treeneg f\) is defined by 
%\begin{equation}
%\label{eqnTreeNeg}
%\treeneg f(x) :=
%\sum_{j_{d + 1} = 1}^{J_{d + 1}}
%a^{(d + 1)}_{1, j_{d + 1}} 
%s_{d}\left( \sum_{j_{d} = 1}^{J_{d}} a^{(d)}_{j_{d + 1}, j_{d}} 
%\right. \\ &\qquad \left. 
%s_{d - 1}\left(  \ldots
%\sum_{j_2 = 1}^{J_2} a_{j_3, j_2}^{(2)} s_{1}\left(\left(a^{(1)}_{j_{2}}\right)^{\trans} x 
%-
%w^{(j_{2: d + 1})}_{f}
%\right)
%\right)
%\right).
%\end{equation}
\label{defTreePosNeg}
\end{definition}

We next provide a bound on the adversarial loss using these tree transforms:

\begin{proposition}
Let \(f\) be a neural network.
Then we have the inequality
\[
\sup_{w \in B(\epsilon)} \lrtrunc(f, z+w)
=
\lrtruncpn(\ptrans f, \ntrans f, z)
\leq 
\lrtruncpn(T_{+}f, T_{-}f, z).
\]
\label{propTreePosNegLoss}
\end{proposition}

Next, we state our usual Rademacher complexity bound:

\begin{corollary}
Let \(\functions_{\nets}\) be a class of neural networks of depth \(d\) satisfying \(\|A_{j}\|_{\infty} \leq \alpha_{j}\) and \(\|A_{j}\|_{F} \leq \alpha_{1, F}\), for each \(j = 1, \ldots, d+1\), and let \(\alpha = \prod_{j = 1}^{d + 1} \alpha_{j}\). 
Additionally, suppose \(\max_{j = 1, \ldots, J_{1}}\|a_{j}^{(1)}\|_{q} \leq \alpha_{1, q}\) and \(\|x_{i}\|_2 \le R\) for all $i$.
Then the Rademacher complexities \(\emprad_{n}(\treepos \functions_{\nets})\) and \(\emprad_{n}(\treeneg \functions_{\nets})\) are both upper-bounded by
\begin{align*}
& \begin{aligned}
\alpha\left( \frac{\alpha_{1, F}}{\alpha_{1}}R + \frac{\alpha_{1, q}}{\alpha_{1}} \epsilon\right) 
\cdot \frac{\sqrt{2d \log 2} + 1}{\sqrt{n}}.
   \end{aligned}
\end{align*}
\label{corNNTreePosNegRadBound}
\end{corollary}

The proof is a brief application of the tree transform \(T\).
Finally, we arrive at a generalization bound:

\begin{corollary}
Let \(\functions_{\nets}\) be a class of neural networks of depth \(d\). 
Let \(g_{i}(a) = \ellbar_{\text{xe}}(\delta(a), y_{i})\).
Under the same assumptions as in Corollary~\ref{corNNTreePosNegRadBound}, for any $f \in \functions_{\nets}$, we have
the upper bounds
\begin{align*}
&\begin{aligned}
\risk_{\robust}(\lrtrunc, f) = \expect_{P} \lrtruncpn (f, z) 
&\leq 
\frac{1}{n} \sum_{i = 1}^{n} \lrtruncpn (f, z_{i})  +
3B^{r}\sqrt{\frac{\log \frac{2}{\delta}}{2n}} 
\\ &\qquad 
+ 
4 r B^{r - 1} \alpha\left( \frac{\alpha_{1, F}}{\alpha_{1}}R + \frac{\alpha_{1, q}}{\alpha_{1}} \epsilon\right) 
%\\ &\qquad  \times 
\frac{\sqrt{2d \log 2} + 1}{\sqrt{n}},
\end{aligned}
\end{align*}
with probability at least $1-\delta$.
\label{corNNRegressionRiskBound}
\end{corollary}

Again, the generalization bound for neural networks in the case of regression differs from the bound in binary classification in exactly the same way as for linear classifiers.
In the Rademacher complexity, we obtain an extra factor of \(2\) from non-monotonicity and a factor of \(rB^{r - 1}\) from the Lipschitz constant.
The bounded differences term, which is the second term in the above generalization bound, has an extra \(B^{r}\) factor since the loss is bounded on \([0, B^{r}]\).
%\textcolor{red}{Compare with classification case.}

%---------------------------------------------%
%---------------------------------------------%
\section{Optimization of risk bounds}
\label{subsecAlgorithms}

Our sample-based upper bounds on adversarial risk suggest the strategy of optimizing the bounds in the corollaries, rather than simply the standard empirical risk, to achieve robustness of the trained networks against adversarial perturbations.
Accordingly, we provide two algorithms for optimizing the upper bounds appearing in Corollary~\ref{corGeneralizationLinear} for linear classifiers. 
Additionally, we briefly discuss algorithms suggested by Corollary~\ref{corNNGeneralization} and comment on the computational difficulties.

%---------------------------------------------%
%---------------------------------------------%
\subsection{Optimization for linear classifiers}

One idea is to optimize the first bound~\eqref{EqnLin1} directly.
Recalling the form of $\Psi$, this leads to the following optimization problem:
\begin{align}
&\begin{aligned}
\min_{\theta, b} \hspace{25pt}
& 
\sum_{i = 1}^{n} 
\max\{0, 1 - y_{i} (\theta^{\trans} x_{i} + b) + \epsilon \|\theta\|_{q}\}.
\label{eqnLinearObj1}
\end{aligned}
\end{align}
Note that the optimization problem of equation~\eqref{eqnLinearObj1} is convex in \(\theta\) and \(b\); therefore, this is a computationally tractable problem.
We summarize this approach in Algorithm~\ref{algConvex}.

%-----------------------------------------------------------%
\begin{algorithm}[!h]
\SetKwInOut{Input}{Input}
\Input{Data \(z_{1}, \ldots, z_{n}\), function class \(\functions_{\linear}\).}

Solve equation~\eqref{eqnLinearObj1} to obtain $(\hat{\theta}, \hat{b})$.

Return the resulting classifier $\sgn(\hat{f})$, where $\hat{f}(x) = \hat{\theta} x + \hat{b}$.

\caption{Convex risk}
\label{algConvex}
\end{algorithm}
%-----------------------------------------------------------%

The second approach involves optimizing the second adversarial risk bound~\eqref{EqnLin2}.
Although this bound is generally looser than the bound~\eqref{EqnLin1}, we remark on optimization due to the fact that regularization is a popular mechanism for encouraging generalization. However, note that the regularization coefficient in the bound~\eqref{EqnLin2} depends on $f$. Thus, we propose to perform a grid search over the value of the regularization parameter.

%This more closely resembles usual regularized problems, but we need to perform a grid search for the regularization coefficient, since it also depends on $f$.
Specifically, define
\begin{equation}
\label{eqnGamma}
\gamma_{\linear}(f)
:=
\sum_{i = 1}^{n} \lhzo(\Psi f, z_{i}).
\end{equation}
We then have the optimization problem
\begin{align}
&\begin{aligned}
\min_{\theta, b} \hspace{14pt}
& 
\sum_{i = 1}^{n} 
\max\{0, 1 - y_{i} (\theta^{\trans} x_{i} + b) \}  
+ \epsilon \|\theta\|_{q} \gamma_{\linear}(f).
\label{eqnLinearObj2}
\end{aligned}
\end{align}
Note, however, that \(\gamma_{\linear}(f)\) is nonconvex, and the form as a function of $\theta$ and $b$ is complicated. We propose to take \(\gamma_{i} = i/n\) for \(i = 0, \ldots, n\) and solve
\begin{align}
&\begin{aligned}
\min_{\theta, b} \hspace{14pt}
& 
\sum_{j = 1}^{n} 
\max\{0, 1 - y_{j} (\theta^{\trans} x_{j} + b) \}  
+ \epsilon \|\theta\|_{q} \gamma_{i}.
\label{eqnLinearObj3}
\end{aligned}
\end{align}
At the end, we simply pick the solution minimizing the objective function in equation~\eqref{eqnLinearObj2} over all $i$. Note that this involves evaluating equation~\eqref{eqnGamma}, but this is easy to do in the linear case.
Note that, when \(q = 2\), equation~\eqref{eqnLinearObj3} is essentially a support vector machine. 
This method is summarized in Algorithm~\ref{algExhaustive}.

%-----------------------------------------------------------%
\begin{algorithm}[!h]
\SetKwInOut{Input}{Input}
\Input{Data \(z_{1}, \ldots, z_{n}\), function class \(\functions_{\linear}\).}
\For{$i = 0, \ldots, n$}{
Set \(\gamma_{i} = i /n\).

Calculate the \(f_{i}\) minimizing equation~\eqref{eqnLinearObj3}.

Save the robust empirical risk, the objective of equation~\eqref{eqnLinearObj2}, of \(f_{i}\) as \(\risk_{i}\). 
%\textcolor{red}{How do we compute this? Do we have to explicitly solve the problem~\eqref{eqnGamma}?}
}

Return the \(f_{i}\) with the minimum \(\risk_{i}\).

\caption{Regularized risk}
\label{algExhaustive}
\end{algorithm}

Both of these algorithms can also be adapted to the cases of multiclass classification and regression.
The main difference is simply the loss function; equation~\eqref{eqnLinearObj1} would be modified to the multiclass hinge loss or the squared-error loss, instead.

%---------------------------------------------%
%---------------------------------------------%
\subsection{Optimization for neural networks}

In the case of neural networks, we will confine our discussion to minimizing the empirical risk of \(Tf\), since it is less clear how to obtain a useful algorithmic problem from the perspective of regularization.

In principle, there is nothing wrong with fixing a level of robustness \(\epsilon\) and then attempting to find an \(f\) minimizing the empirical risk of \(Tf\).
In practice, this problem becomes computationally difficult for relatively small neural networks.
A major problem is computing the perturbations \(w_{f}^{(j_{1 : d + 1})}\). 
In particular, we need to compute 
\(\prod_{k = 2}^{d + 1} J_{k}\) different perturbations for a fully-connected neural network each time we wish to evaluate \(Tf\).
For smaller networks, this can be done, and 
we formalize this as Algorithm~\ref{algNNdirect}.

%---------------------------------------------%
\begin{algorithm}[!h]
\SetKwInOut{Input}{Input}
\Input{Data \(z_{1}, \ldots, z_{n}\), function class \(\functions_{\nets}\).}

Find an \(\hat{f}\) minimizing the empirical risk of \(Tf\).

Return the resulting neural network \(\hat{f}\).
\caption{Empirical risk of \(Tf\)}
\label{algNNdirect}
\end{algorithm}
%---------------------------------------------%

For large networks, this computational cost is prohibitive. As a result, we would like to restrict our analysis to a class of neural networks \(\functions_{\nets}^{\star}\) and a transformation \(T^{\star}\) such that these perturbations can be computed faster.

To this end, we suggest a few additional strategies to improve the speed of computing \(w_{f}^{(j_{1: d + 1})}\). The first idea is to fix the signs of the matrix entries in a layer. The second idea is to bound the \(q\)-norm of the rows in the first-layer matrix \(A^{(1)}\). Suppose all of the \(q\)-norms are bounded by \(c_{q}\). Then we could replace \(\|a_{j_{2}}^{(1)}\|_q\) by \(c_{q}\) in the computation of \(w_{f}^{(j_{1: d + 1})}\) to obtain a new perturbation \(w_{q}^{(j_{1: d + 1})}\) and a new transform \(T^{\star} f\). Note that with the above two modifications, the perturbation still depends on the path through the neural network, but it no longer depends on the exact network \(f\). Thus, a third idea is to reduce the number of paths through the network that lead to different perturbations. The drawback to these approaches is a possible decrease in prediction accuracy, since the restricted class of neural networks may not be able to approximate the requisite classification functions as accurately.

As an example, suppose \(f: \xspace \to \reals\) is a neural network of depth \(d\) for binary classification. Suppose all the matrices \(A^{(1)}, \ldots, A^{(d)}\) are constrained to be nonnegative. 
Furthermore, suppose \(J_{d} = 2\), and that \(a^{(d + 1)}_{1, 1}\) is nonnegative and \(a^{(d + 1)}_{1, 2}\) is nonpositive. 
Additionally, let \(\|a_{j}^{(1)}\|_{q} \leq c_{q}\) for all \(j\).
Denote the depth \(d - 1\) sub-network of \(f\) by \(f^{(d - 1)}: \xspace \to \reals^{2}\). Then we define
\begin{align*}
& \begin{aligned}
T^{\star} f(x, y)
&=
a^{(d + 1)}_{1, 1} 
s_{d}(T^{\star}f^{(d - 1)}_{1}(x, y))
+
a^{(d + 1)}_{1, 2} 
s_{d}(T^{\star}f^{(d - 1)}_{2}(x, -y)),
   \end{aligned}
\end{align*}
where for \(k = 1, 2\), we define
\begin{align*}
& \begin{aligned}
T^{\star} f^{(d - 1)}_{k}(x, y) 
&=
A^{(d)} s_{d - 1}\left(\ldots s_{1}\left(A^{(1)}x + w_{q}^{(k)} \ind \right)\right),
   \end{aligned}
\end{align*}
where \(\ind\) denotes the all-ones vector and \(w_{q}^{(k)} = - y \epsilon c_{q}\). 
Crucially, we can compute \(T^{\star} f_{k}^{(d - 1)}\) by straightforward matrix multiplications once \(w_{q}^{(k)}\) has been computed.
On the downside, since such a setup would lead to a constrained optimization problem where \(A^{(j)} \geq 0\) element-wise for all \(j \leq d\), \(a_{1, 1}^{(d + 1)} \geq 0\), and \(a_{1, 2}^{(d + 1)} \leq 0\), an optimization algorithm that can handle constraints would be necessary instead of the usual unconstrained algorithms employed in optimizing neural networks.

Finally, we note that the preceding discussion applies to the regression case with little modification. 
The main difference is again that different loss functions should be used, and for multiclass classification \(f\) should map to \(\reals^{K}\) instead of \(\reals\).

%\textcolor{red}{What about optimization in the regression case?}

%---------------------------------------------%
%---------------------------------------------%

%---------------------------------------------%
%---------------------------------------------%

\section{Binary classification proofs}
\label{secProofs}

In this section, we provide the proofs of the main results in the case of binary classification using linear or neural networks.

\subsection{Sup and tree transforms}
\label{SecSupTree}

We first present the proofs of our core theoretical results regarding the transform functions $\Psi$ and $T$.

\begin{proof}[Proof of Proposition~\ref{propSupTransformEquality}]
We break our analysis into two cases.
If \(y = +1\), then \(\ellbar(f(x), +1)\) is decreasing in \(f(x)\).
Thus, we have
\begin{align*}
\sup_{w \in B(\epsilon)} \ellbar(f(x + w), +1)  &=
\ellbar\left(\inf_{w \in B(\epsilon)} f(x+w), +1\right) =
\ellbar\left((-1) \sup_{w \in B(\epsilon)} (-1) f(x+w), +1\right) \\
& = \ell\left(\Psi f, (x, +1)\right).
\end{align*}
If instead $y = -1$, then \(\ellbar(f(x), -1)\) is increasing in \(f(x)\), so
\begin{align*}
\sup_{w \in B(\epsilon)} \ellbar(f(x+w), -1) &=
\ellbar\left(\sup_{w \in B(\epsilon)} f(x+w), -1\right) = \ellbar\left( (1) \sup_{w \in B(\epsilon)} (1) f(x+w), -1\right) \\
& = \ell\left(\Psi f, (x, -1)\right).
\end{align*}
This completes the proof.
\end{proof}

%---------------------------------------------%
%---------------------------------------------%

\begin{proof}[Proof of Proposition~\ref{propSupLinear}]
Using the definition of the sup transform, we have
\begin{align*}
& \begin{aligned}
\Psi f(x, y) 
&=
-y \sup_{w \in B(\epsilon)} (-y)(\theta^{\trans} x + b + \theta^{\trans} w) \\
&=
\theta^{\trans} x + b - y \sup_{w \in B(\epsilon)} (-y) \theta^{\trans} w \\ 
&=
\theta^{\trans} x + b - y \epsilon \|\theta\|_{q},
   \end{aligned}
\end{align*}
where the final equality comes from the variational definition of the \(\ell_{q}\)-norm.
This completes the proof.
\end{proof}

%---------------------------------------------%
%---------------------------------------------%
Before we begin the proof of Proposition~\ref{propTreeLoss}, we state, prove, and remark upon a helpful lemma.
We want to apply this iteratively to push the supremum inside the layers of the neural network.

\begin{lemma}
Let \(g: \xspace \to \reals^{J}\) be a function and define \(s: \reals \to \reals\) to be a monotonically increasing function applied elementwise to vectors.
Then we have the inequality
\begin{align*}
&\begin{aligned}
\sup _{w \in B(\epsilon)}& 
\sum_{j = 1}^{J} b_{j} s(a_{j}^{\trans} g(x + w))
\leq \sum_{j = 1}^{J} b_{j} 
%\\ &\times  
s\left(\sgn(b_{j}) \sup_{w^{(j)} \in B(\epsilon)} \sgn(b_{j}) a_{j}^{\trans} g\left(x + w^{(j)}\right)\right).
\end{aligned}
\end{align*}
\label{lemmaTreeIterate}
\end{lemma}

\begin{proof}
Denote the left hand-side of the desired inequality by \(L\).
First, we can push the supremum inside the sum to obtain
\[
L
\leq 
\sum_{j = 1}^{J} \sup_{w^{(j)} \in B(\epsilon)} 
b_{j} s\left(a_{j}^{\trans} g\left(x + w^{(j)}\right)\right).
\]
%where here we split the choice of \(w\) into the choice of \(w_{j}\).
%At this point, we need to push the supremum inside of \(s\).
Next, note that
\begin{align}
&\begin{aligned}
\label{EqnInter}
\sup_{w^{(j)} \in B(\epsilon)} &
b_{j} s\left(a_{j}^{\trans} g\left(x + w^{(j)}\right)\right)
= 
\sup_{w^{(j)} \in B(\epsilon)} 
b_{j} s\left(\sgn(b_{j}) \sgn(b_{j}) a_{j}^{\trans} g\left(x + w^{(j)}\right)\right).
\end{aligned}
\end{align}
Since \(s\) is monotonically increasing, we see that the map
\(x \mapsto b_{j} s( \sgn(b_{j}) x)\) is monotonically increasing, as well.
Thus, the supremum in equation~\eqref{EqnInter} is obtained when $\sgn(b_j)a_j^{\trans} g(x+w^{(j)})$ is maximized. Hence, we obtain
\begin{align*}
&\begin{aligned}
L
&\leq 
\sum_{j = 1}^{J} b_{j}  s\left(\sgn(b_{j}) \sup_{w^{(j)} \in B(\epsilon)} \sgn(b_{j}) a_{j}^{\trans} g\left(x + w^{(j)}\right)\right),
\end{aligned}
\end{align*}
which completes the proof.
\end{proof}

\begin{remark}
\label{RemIterate}
Note that if \(f(x) = b^{\trans} s(A g(x))\), where \(g(x) = s'(A' h(x))\),
this lemma yields 
%\textcolor{red}{why is there a factor of $-y$ in the following inequalities?}
\begin{align*}
& \begin{aligned}
L
&\leq 
\sum_{j = 1}^{J} b_{j} s\left(\sgn(b_{j})
\sup_{w^{(j)} \in B(\epsilon)} 
% \right. \\ &\qquad \times \left.
\sum_{k = 1}^{K} \sgn(b_{j}) a_{j, k} 
s'\left( (a'_{k})^{\trans} h\left(x + w^{(j)}\right) 
\right)
\right).
   \end{aligned}
\end{align*}
If we apply Lemma~\ref{lemmaTreeIterate} again, we obtain
\begin{align*}
& \begin{aligned}
L 
&\leq 
\sum_{j = 1}^{J} b_{j} s\left(\sgn(b_{j})
\sum_{k = 1}^{K} \sgn(b_{j}) a_{j, k} 
% \right. \\ &\qquad \times \left.
s'\left( \sgn(b_{j} a_{j, k}) \sup_{w^{(j, k)} \in B(\epsilon)} \sgn(b_{j} a_{j, k}) 
%\right. \right. \\ &\qquad \times \left. \left.
(a'_{k})^{\trans} h\left(x + w^{(j, k)}\right) 
\right)
\right)
\\
&=
\sum_{j = 1}^{J} b_{j} s\left(
\sum_{k = 1}^{K}  a_{j, k} 
% \right. \\ &\qquad \times \left.
s'\left(\sgn(b_{j} a_{j, k}) \sup_{w^{(j, k)} \in B(\epsilon)} 
\sgn(b_{j} a_{j, k}) 
% \right. \right. \\ &\qquad \times \left. \left.
(a'_{k})^{\trans} h\left(x + w^{(j, k)}\right) 
\right)
\right).
   \end{aligned}
\end{align*}
In particular, we note  that the sign terms accumulate within the supremum, but when we take the supremum inside another layer, the sign terms $\sgn(b_j)$ remaining in the previous layers cancel out and are incorporated into the \(\sgn(b_{j} a_{j, k})\) of the next layer.
\end{remark}

\begin{proof}[Proof of Proposition~\ref{propTreeLoss}]
First note that the assumption that \(\ell\) is monotonically decreasing in \(y f(x)\) is equivalent to \(\ell\) being monotonically increasing in \(-y f(x)\). As in the proof of Proposition~\ref{propSupTransformEquality}, if $y=+1$, we want to show that $\Psi f(x,y) \ge Tf(x,y)$; if $y = -1$, we want to show that $\Psi f(x,y) \le Tf(x,y)$.
Thus, it is our goal to establish the inequality
\begin{equation}
-y \Psi f(x, y) 
\leq 
-y T f(x, y).
\end{equation}
We define \(L := -y \Psi f(x, y)\) and  
show how to take the supremum inside each layer of the neural network to yield \(-y T f(x, y)\).
To this end, we simply apply Lemma~\ref{lemmaTreeIterate} and Remark~\ref{RemIterate} iteratively until the remaining function is linear.
Thus, we see that 
\begin{align*}
& \begin{aligned}
L
&\leq 
-y \sum_{j_{d + 1} = 1}^{J_{d + 1}} a^{(d + 1)}_{1, j_{d + 1}} 
s_{d}\left( \sum_{j_{d} = 1}^{J_{d}} a^{(d)}_{j_{d + 1}, j_{d}}
% \right. \\ &\qquad  \left. \times 
s_{d - 1}\left( \sum_{j_{d - 1} = 1}^{J_{d - 1}} a^{(d - 1)}_{j_{d}, j_{d - 1}} 
s_{d - 2}\left(\ldots
%\right. \right. \right. \\ &\qquad  \left. \left. \left. \times 
s_{1}\left(
\sgn\left(-y a_{1}^{(d + 1)} a_{1, j_{d}}^{(d)} \ldots a_{j_{3}, j_{2}}^{(2)}\right)
 \right. \right. \right. \right. \\ & \qquad   \left. \left. \left. \left.
\times  
\sup _{w^{(j_{2: d + 1})} \in B(\epsilon)}
\sgn\left(-y a_{1}^{(d + 1)} a_{1, j_{d}}^{(d)} \ldots a_{j_{3}, j_{2}}^{(2)}\right)
%\right. \right. \right. \right. \\ &\qquad \left. \left. \left. \left.
%\times 
\left(a^{(1)}_{j_{2}}\right)^{\trans} 
\left(x 
+
w^{(j_{2: d + 1})}\right)
\right)
\right)
\right)
\right), 
   \end{aligned}
\end{align*}
and simplifying gives
\begin{align*}
& \begin{aligned}
L
&\leq
-y \sum_{j_{d + 1} = 1}^{J_{d + 1}} a^{(d + 1)}_{1, j_{d + 1}}  
s_{d}\left( \sum_{j_{d} = 1}^{J_{d}} a^{(d)}_{j_{d + 1}, j_{d}}
%\right. \\ &\qquad  \left. 
%\times 
s_{d - 1}\left( \sum_{j_{d - 1} = 1}^{J_{d - 1}} a^{(d - 1)}_{j_{d}, j_{d - 1}} 
s_{d - 2}\left(\ldots
%\right.\right. \right. \\ &\qquad \left. \left. \left.
%\times  
s_{1}\left(
\left(a^{(1)}_{j_{2}}\right)^{\trans} 
x 
\right. \right. \right. \right. \\ & \qquad \left. \left. \left. \left. +
\sgn\left(-y a_{1}^{(d + 1)} a_{1, j_{d}}^{(d)} \ldots a_{j_{3}, j_{2}}^{(2)}\right)
\right. \right. \right. \right. \\
& \qquad  \left. \left. \left. \left. \left.
\times 
\sup _{w^{(j_{2: d + 1})} \in B(\epsilon)}
\sgn\left(-y a_{1}^{(d + 1)} a_{1, j_{d}}^{(d)} \ldots a_{j_{3}, j_{2}}^{(2)}\right)
\left(a^{(1)}_{j_{2}}\right)^{\trans} 
w^{(j_{2: d + 1})}
\right)
\right)
\right)
\right) 
\right).
   \end{aligned}
\end{align*}
The final supremum clearly evaluates to $\epsilon \|a_{j_2}^{(1)}\|_q$.
%\begin{align*}
%&\begin{aligned}
%w_{f}^{(j_{2: d + 1})} 
%&:=
%\left(a_{j_{2}}^{(1)}\right)^{\trans}
%w^{(j_{2: d + 1})} 
% \\ & \qquad \times
%\sgn\left(-y a_1^{(d + 1)} a_{1, j_{d}}^{(d)} \ldots a_{j_{3}, j_{2}}^{(2)} \right) \\
%&=
%\epsilon \left\|a_{j_{2}}^{(1)}\right\|_{q}
%\sgn\left(-y a_1^{(d + 1)} a_{1, j_{d}}^{(d)} \ldots a_{j_{3}, j_{2}}^{(2)} \right).
%\end{aligned}
%\end{align*}
%Thus,
Recalling the definition~\eqref{EqnThreeStar} of $w_f^{(j_{2:d+1})}$, we then have
\begin{align*}
& \begin{aligned}
-y \Psi f(x, y)
&\leq 
-y \sum_{j_{d + 1} = 1}^{J_{d + 1}} a^{(d + 1)}_{1, j_{d + 1}} 
s_{d}\left( \sum_{j_{d} = 1}^{J_{d}} a^{(d)}_{j_{d + 1}, j_{d}}
s_{d - 1}\left( 
%\right. \right. \\ &\qquad \left. \left. 
\ldots  
s_{1}\left(
\left(a^{(1)}_{j_{2}}\right)^{\trans} 
x 
+
w^{(j_{2: d + 1})}_{f}
\right)
\right)
%\right)
\right) \\
&=
-y T f(x, y),
   \end{aligned}
\end{align*}
which proves the proposition.
\end{proof}

\subsection{Rademacher complexity proofs}
\label{subsecRademacher}

In this section, we prove Lemmas~\ref{lemmaLinearERC} and \ref{lemmaNNRad2}, which are the bounds on the empirical Rademacher complexities of \(\Psi \functions_{\linear}\) and \(T \functions_{\nets}\).
The proofs are largely based on preexisting proofs for bounding the empirical Rademacher complexities of \(\functions_{\linear}\) and \(\functions_{\nets}\), and this simplicity is part of what makes \(\Psi\) and \(T\) attractive.

\begin{proof}[Proof of Lemma~\ref{lemmaLinearERC}]
Using Proposition~\ref{propSupLinear}, we have
\begin{align*}
&\begin{aligned}
n \emprad_{n}(\Psi \functions_{\linear})
&=
\expect_{\sigma} \left[\sup_{f \in \functions} \sum_{i = 1}^{n} \sigma_{i} \Psi f(x_i, y_i)\right] \\
&=
\expect_{\sigma} \left[\sup_{f \in \functions} \sum_{i = 1}^{n} \sigma_{i} \left(\theta^{\trans}x_{i} + b - y_{i} \epsilon \|\theta\|_{q} \right)\right] \\
&\leq 
\expect_{\sigma} \left[\sup_{f \in \functions} \sum_{i = 1}^{n} \sigma_{i} \left(\theta^{\trans}x_{i} + b \right)\right] \\
&\qquad 
+
\expect_{\sigma} \left[\sup_{f \in \functions} \sum_{i = 1}^{n} \sigma_{i} 
\left(- y_{i} \epsilon \|\theta\|_{q} \right)\right] \\
&=
n \emprad_{n}( \functions_{\linear})
+
\epsilon\expect_{\sigma} \left[\sup_{f \in \functions}  \|\theta\|_{q} \sum_{i = 1}^{n} \sigma_{i}\right].
\end{aligned}
\end{align*}
%Note that the inequality comes just from splitting the supremum.
By Lemma~\ref{lemmaLinearRademacher}, the empirical Rademacher complexity of a linear function class is given by
\[
\emprad_{n}(\functions_{\linear})
\leq 
\frac{M_{2} R}{\sqrt{n}}.
\]
Thus, it remains to analyze the second term in the upper bound.

If the sum of the \(\sigma_{i}\)'s is negative, the \(\theta\) maximizing the supremum is the zero vector. 
Alternatively, if the sum is positive, we clearly have the upper bound $M_q \sum_{i=1}^n \sigma_i$. Thus, we have
\begin{align*}
\epsilon\expect_{\sigma} \left[ \sup_{f \in \functions}  \|\theta\|_{q} \sum_{i = 1}^{n} \sigma_{i} \right] 
&\le
\epsilon \expect_{\sigma} \left[M_{q} 
\sum_{i = 1}^{n} \sigma_{i} \ind\left\{\sum_{i = 1}^{n} \sigma_{i} > 0\right\}\right] \\
&\stackrel{(a)}{=}
\frac{\epsilon M_{q}}{2} 
\expect
\left|\sum_{i = 1}^{n} \sigma_{i}\right| \\
&\le
\frac{\epsilon M_{q}}{2} 
\left(\expect\left[
\left(\sum_{i = 1}^{n} \sigma_{i}\right)^{2}
\right]\right)^{\frac{1}{2}},
\end{align*}
where $(a)$ follows because $\sigma_i$ and $-\sigma_i$ have the same distribution, and the last inequality follows by Jensen's inequality. The last term is equal to
$\frac{\epsilon M_{q}}{2} \sqrt{n}$,
using the fact that the \(\sigma_{i}\)'s are independent, zero-mean, and unit-variance random variables.
Putting everything together yields
\[
\emprad_{n}(\Psi \functions_{\linear})
\leq 
\frac{M_{2} R}{\sqrt{n}}
+
\frac{\epsilon M_{q}}{2 \sqrt{n}}, 
\]
which completes the proof.
\end{proof}

%---------------------------------------------%
%---------------------------------------------%
\begin{proof}[Proof of Lemma~\ref{lemmaNNRad2}]
Our broad goal is to peel off the layers of the neural network one at a time.
Most of the work is done by Lemma~\ref{lemmaRademacherPeeling}.
The proof is essentially the same as the Rademacher complexity bounds on neural networks of \cite{golowich2018} until we reach the underlying linear classifier. We then bound the action of the adversary in an analogous manner to the linear case.

We write
\begin{align*}
& \begin{aligned}
n\emprad(T \functions_{\nets}) 
&=
\frac{1}{\lambda} \log \exp\left(\lambda 
\expect\left[\sup_{f \in \functions_{\nets}} \sum_{i = 1}^{n} \sigma_{i} Tf(x_{i}, y_i)\right] 
\right)  \\
&\leq 
\frac{1}{\lambda} \log
\expect\left[\sup_{f \in \functions_{\nets}} \exp\left(\lambda\sum_{i = 1}^{n} \sigma_{i} Tf(x_{i}, y_i) 
\right)\right].
   \end{aligned}
\end{align*}

Recalling the form of $Tf$ from equation~\eqref{EqnTwoStar}, we can apply Lemma~\ref{lemmaRademacherPeeling} successively \(d\) times with 
\(G(x) = \exp((\lambda \prod_{j \in J} \alpha_{j}) x)\) for various \(J\) in order to remove the layers of the neural network. 
Specifically, we use \(J = \emptyset\), \(J = \{d + 1\}\), \(J = \{d + 1, d\}\), up to \(J = \{d + 1, \ldots, 3\}\), as we peel away the layers and retain the bounds \(\alpha_{j}\) on the matrix norms from the layers that we have removed.
%\textcolor{red}{I don't quite follow. Can you write out the function $G$ for each case?}
This implies
\begin{align*}
& \begin{aligned}
& n \emprad(\functions_{\nets})  
\leq 
\frac{1}{\lambda} \log 2^{d} \expect \Bigg[\sup _{f \in \functions_{\nets}} \max_{j_{2}, \ldots, j_{d + 1}}
% \\ &\qquad 
\exp\left(\frac{\alpha \lambda}{\alpha_{1}}
\left(a_{j_{2}}^{(1)}\right)^{\trans} \sum_{i = 1}^{n} \sigma_{i} x_{i} +
% \right. \\  &\qquad \left. + 
\frac{\alpha \lambda}{\alpha_{1}} \sum_{i = 1}^{n} \sigma_{i} w_{f}^{(j_{2: d + 1})}
\right) \Bigg] \\ 
& \qquad =
\frac{1}{\lambda} \log 2^{d} \expect \Bigg[\sup _{f \in \functions_{\nets}} \max_{j_{2}, \ldots, j_{d + 1}} 
%\\ &\qquad 
\exp\left(\frac{\alpha \lambda}{\alpha_{1}}
\left(a_{j_{2}}^{(1)}\right)^{\trans} \sum_{i = 1}^{n} \sigma_{i} x_{i} - 
\frac{\alpha \lambda}{\alpha_{1}} \epsilon \left\|a_{j_{2}}^{(1)}\right\|_{q} \sgn(f, j_{2: d + 1})
\sum_{i = 1}^{n} \sigma_{i} y_{i} 
\right)\Bigg].
   \end{aligned}
\end{align*}
%Somewhat miraculously, the adversarial perturbation only amounts to a choice of \(\sgn(f, j_{2: d + 1})\), which is chosen to be the same sign as \(\sum_{i = 1}^{n} \sigma_{i} y_{i}\).
Note that the maxima over \(j_{2}, \ldots, j_{d + 1}\) are accumulated from each application of Lemma~\ref{lemmaRademacherPeeling}.
These maxima correspond to taking a worst-case path through the tree.
To bound the first term, we apply the Cauchy-Schwarz inequality. 
To bound the second term, we use the inequality
\[
-\sgn(f, j_{2: d + 1})
\sum_{i = 1}^{n} \sigma_{i} y_{i} 
\leq 
\left|\sum_{i = 1}^{n} \sigma_{i} y_{i}\right|.
\]
Thus, we have
\begin{align*}
& \begin{aligned}
n \emprad(\functions_{\nets})
&\leq 
\frac{1}{\lambda } \log 2^{d} \expect \Bigg[ \sup_{f \in \functions_{\nets}} \max_{j_{2}, \ldots, j_{d + 1}} 
%\\ & \qquad
\exp\left(\frac{\alpha \lambda}{\alpha_{1}} \left\|a_{j_{2}}^{(1)}\right\|_{2} \left\|\sum_{i = 1}^{n} \sigma_{i} x_{i}\right\|_{2}
%\right. \\ & \qquad \left.
+ 
\frac{\alpha \lambda}{\alpha_{1}} \epsilon \left\|a_{j_{2}}^{(1)}\right\|_{q} \left|\sum_{i = 1}^{n} \sigma_{i} y_{i}\right|
\right) \Bigg]\\
&\leq 
\frac{1}{\lambda } \log 2^{d} \expect \Bigg[
\exp\left(\frac{\alpha \alpha_{1, F} \lambda}{\alpha_{1}} \left\|\sum_{i = 1}^{n} \sigma_{i} x_{i}\right\|_{2}
%\right. \\ & \qquad \left.
+ 
\frac{\alpha \alpha _{1, q}}{\alpha _{1}} \lambda \epsilon \left|\sum_{i = 1}^{n} \sigma_{i} y_{i}\right|
\right)\Bigg].
   \end{aligned}
\end{align*}

In order to bound the final expectation, we define
\[
Z(\sigma) :=
\alpha \left(
\frac{\alpha_{1, F} }{\alpha_{1}} \left\|\sum_{i = 1}^{n} \sigma_{i} x_{i}\right\|_{2}
+ 
\frac{\alpha_{1, q}}{\alpha_{1}}  \epsilon \left|\sum_{i = 1}^{n} \sigma_{i} y_{i}\right| \right),
\]
where we view \(Z\) as a function of the \(\sigma_{i}\)'s.
Now we have
\begin{align}
& \begin{aligned}
n\emprad(\functions_{\nets}) 
&\leq 
\frac{1}{\lambda} d \log 2 
%\\ &\qquad 
+ 
\log \expect\left[\exp(\lambda(Z - \expect\left[Z\right])) \exp(\lambda \expect\left[Z\right])\right] \\ 
&=
\frac{1}{\lambda} d \log 2 
%\\ &\qquad 
+ 
\frac{1}{\lambda}\log \expect\left[\exp(\lambda(Z - \expect\left[Z\right]))\right] 
+
 \expect[Z].
\label{eqnLogDecomp}
   \end{aligned}
\end{align}
Thus, it remains to compute the last two terms on the right-hand side.
We start with the expectation of \(Z\).
By Jensen's inequality,
%\begin{align*}
%& \begin{aligned}
%\expect\left[Z\right] 
%&=
%\alpha\frac{\alpha_{1, F}  }{\alpha_{1}} \expect \left\|\sum_{i = 1}^{n} \sigma_{i} x_{i}\right\|_{2}
%+ 
%\alpha \frac{\alpha_{1, q}}{\alpha_{1}} \epsilon \expect \left|\sum_{i = 1}^{n} \sigma_{i} y_{i}\right|.
%   \end{aligned}
%\end{align*}
%For each of the terms in expectations, we use Jensen's inequality to obtain
\begin{align*}
& \begin{aligned}
\expect\left[Z\right]  
&\leq 
\alpha \frac{\alpha_{1, F}}{\alpha_{1}} 
\left(\expect \left\|\sum_{i = 1}^{n} \sigma_{i} x_{i}\right\|_{2}^{2}
\right)^{\frac{1}{2}}
% \\ & \qquad 
+ 
\alpha \frac{\alpha_{1, q}}{\alpha_{1}} \epsilon 
\left(
\expect \left(\sum_{i = 1}^{n} \sigma_{i} y_{i}\right)^{2}
\right)^{\frac{1}{2}} \\
&=
\alpha \frac{\alpha_{1, F} }{\alpha_{1}} 
\left(\expect \sum_{i, j = 1}^{n} \sigma_{i} \sigma_{j} x_{i}^{\trans} x_{j}
\right)^{\frac{1}{2}}
+ 
\alpha \frac{\alpha_{1, q}}{\alpha_{1}} \epsilon 
n^{\frac{1}{2}} \\
&=
\alpha \frac{\alpha_{1, F} }{\alpha_{1}} 
\left(\sum_{i = 1}^{n} \|x_{i}\|^{2}
\right)^{\frac{1}{2}}
+ 
\alpha \frac{\alpha_{1, q}}{\alpha_{1}} \epsilon 
n^{\frac{1}{2}} \\
&\le
\alpha 
\left(\frac{\alpha_{1, F} }{\alpha_{1}} 
R 
+ 
\frac{\alpha_{1, q}}{\alpha_{1}} \epsilon 
\right)
\sqrt{n} \\
&:=
C \sqrt{n}.
   \end{aligned}
\end{align*}

Next, we need to handle the middle term in inequality~\eqref{eqnLogDecomp}.
The idea is to use standard bounds employed in concentration inequalities.
Let \(\sigma_{i}' = \sigma_{i}\) for all \(i = 1, \ldots, n\), except for one, where \(\sigma_{j}' = - \sigma_{j}\).
Treating \(Z\) as a function of the \(\sigma_{i}\)'s, we obtain
\begin{align*}
& \begin{aligned}
Z(\sigma)
-
Z(\sigma')
&=
\frac{\alpha \alpha_{1, F}}{\alpha_{1}} 
\left(\left\|\sum_{i = 1}^{n} \sigma_{i} x_{i}\right\|_{2}  
- \left\|\sum_{i = 1}^{n} \sigma_{i}' x_{i}\right\|_{2}\right)
%\\ & \qquad 
+
\frac{\alpha \alpha_{1, q}}{\alpha_{1}} \epsilon 
\left(\left| \sum_{i = 1}^{n} \sigma_{i} y_{i}\right| - \left|\sum_{i = 1}^{n} \sigma_{i}'y_{i}\right| \right) \\
&\leq 
2\alpha\left(\frac{ \alpha_{1, F}}{\alpha_{1}} \left\|x_{j}\right\|_{2} 
+
\frac{ \alpha_{1, q}}{\alpha_{1}} \epsilon\right) \\
&\leq 
2C.
   \end{aligned}
\end{align*}
Thus, the variance factor in the bounded differences inequality (Lemma~\ref{lemmaTalagrandsLemma1991}) is
\[
v 
=
\frac{1}{4} \sum_{i = 1}^{n} (2C)^{2}
=
C^{2}n.
\]
This yields
\[
\frac{1}{\lambda} \log \expect \exp(\lambda(Z - \expect\left[Z\right]))
\leq 
\frac{1}{\lambda} \cdot \frac{\lambda^{2} C^{2} n}{2} \\
=
\frac{\lambda C^{2} n}{2}.
\]

Finally, putting everything together, we have
\begin{align*}
& \begin{aligned}
n \emprad(T \functions_{\nets})
&\leq 
\frac{1}{\lambda} d\log 2 + \frac{\lambda C^{2} n}{2} + C \sqrt{n} 
% \\
&=
C \left(\sqrt{2 d \log 2} + 1\right) \sqrt{n}, 
   \end{aligned}
\end{align*}
where in the last equality, we set the free parameter to be  \(\lambda = (2 d \log(2) / (C^{2} n))^{1/2}\) to minimize the bound. This completes the proof.
\end{proof}

\subsection{Proofs of corollaries}
\label{subsecCorProofs}

\begin{proof}[Proof of Corollary~\ref{corGeneralizationLinear}]
%Define the bounded hinge loss \(\lhinge'(f, z) = \min\{1, \lhinge(f, z)\}\).
%This bounded hinge loss also bounds the indicator loss, so we have 
%\[
%\expect _{\dist} \lzo(\Psi f, z) \leq \expect _{\dist} \lhinge'(\Psi f, z).
%\]
Applying Lemma~\ref{thmMainGeneralization} in Appendix~\ref{appAuxLemmas} and the Rademacher complexity bound of Lemma~\ref{lemmaLinearERC} immediately gives
\begin{align*}
& \begin{aligned}
\expect _{\dist} \ell(\Psi f, z) 
&\leq 
\frac{1}{n} \sum_{i = 1}^{n} \ell(\Psi f, z_{i})
+ 2 \frac{M_{2} R}{\sqrt{n}} 
% \\ &\qquad 
+ \frac{\epsilon M_{q}}{\sqrt{n}} + 3 \sqrt{\frac{\log\frac{2}{\delta}}{2n}}. 
   \end{aligned}
\end{align*}
This is the first of the two bounds that we wished to prove.
To prove the second bound, we simply apply this result to the bounded hinge loss \(\lhinge'(f, z) = \min\{1, \lhinge(f, z)\}\) and then use Lemma~\ref{lemmaLinearEmpiricalBound}.
\end{proof}

%---------------------------------------------%
%---------------------------------------------%
\begin{proof}[Proof of Corollary~\ref{corNNGeneralization}]
%We define the bounded cross-entropy loss to be \(\lxe'(f, z) = \min\{1, \lxe(f, z)\}\).
%Then the bounded cross-entropy loss also bounds the indicator loss, so we have
%\[
%\expect _{\dist} \lzo(f, z)
%\leq 
%\expect _{\dist} \lxe'(f, z).
%\]
Applying Lemma~\ref{thmMainGeneralization} and the Rademacher complexity bound of Lemma~\ref{lemmaNNRad2} gives
\begin{align*}
& \begin{aligned}
\expect_{P} \ell (Tf, z) 
&\leq 
\frac{1}{n} \sum_{i = 1}^{n} \ell(Tf, z_{i}) 
+
3\sqrt{\frac{\log \frac{2}{\delta}}{2n}}
%\\ & \qquad 
+
2 \alpha\left( \frac{\alpha_{1, F}}{\alpha_{1}}R + \frac{\alpha_{1, q}}{\alpha_{1}} \epsilon\right) 
% \\ &\qquad \times
 \frac{\sqrt{2d \log 2} + 1}{\sqrt{n}}
% \\
%&\leq 
%\frac{1}{n} \sum_{i = 1}^{n} \lxe(Tf, z_{i}) 
%+
%3\sqrt{\frac{\log \frac{2}{\delta}}{2n}}
% \\ &\qquad 
%+ 
%2 \alpha\left( \frac{\alpha_{1, F}}{\alpha_{1}}R + \frac{\alpha_{1, q}}{\alpha_{1}} \epsilon\right) 
% \\ &\qquad \times 
%\frac{\sqrt{2d \log 2} + 1}{\sqrt{n}}.
   \end{aligned}
\end{align*}
This is the first desired generalization bound.
To obtain the second bound, we simply apply this first bound using the bounded cross-entropy loss  \(\lxe'(f, z) = \min\{1, \lxe(f, z)\}\)
%simply note that the cross-entropy loss composed with the softmax is \(1\)-Lipschitz by Lemma~\ref{lemmaCrossEntropyLipschitz}.
%Additionally, the cross-entropy loss satisfies the condition that \(g_{i}(a) \leq g_{i}(b)\) then implies 
%\(|g'_{i}(a)| \leq |g'_{i}(b)|\).
and then apply Lemma~\ref{lemmaNNPER2}.
\end{proof}

%---------------------------------------------%
%---------------------------------------------%
%\section{Numerical Results}
%\label{secNumerical}

%---------------------------------------------%
%---------------------------------------------%
\section{Discussion}
\label{secDiscussion}
We have presented a new method of transforming binary classifiers to obtain upper bounds on the adversarial risk. We have shown that bounding the generalization error of the transformed classifiers may be performed using similar machinery for obtaining traditional generalization bounds in the case of linear classifiers and neural network classifiers. In particular, since the Rademacher complexity of neural networks only has a small additional term due to adversarial perturbations, generalization even in the presence of adversarial perturbations should not  be impossibly difficult for binary classification. Furthermore, we have shown how to extend the results for binary classification to multiclass classification and regression.
%Consequently, poor generalization performance caused by adversarial perturbations is likely to be the result of the failure to minimize the adversarial empirical risk, or the large weights required in practical networks, which causes our bounds to be very large.

We now mention several future directions for research. First, one might be interested in extending the supremum transformation to other types of classifiers. The most interesting avenues would include calculating explicit representations as in the case of linear classifiers, suitable alternative transformations as in the case of neural networks, and bounds on the resulting Rademacher complexities. A second direction is to explore the tree transformation better and develop algorithms for optimizing the resulting adversarial risk bounds. Much of this would be experimental, and we expect that minor tweaks could greatly improve performance in terms of training time, memory usage, and accuracy. One remaining theoretical problem is to develop generalization bounds for more sophisticated networks. Here, we have only studied feed-forward neural networks with the requisite activation functions. Notably, this is the broadest class of networks for which provable bounds for the adversarial risk currently exist.

%---------------------------------------------%
%---------------------------------------------%
\bibliographystyle{abbrvnat}
\bibliography{refs}

\newpage
\appendix

%---------------------------------------------%
%---------------------------------------------%

%---------------------------------------------%
%---------------------------------------------%
\section{Additional lemmas}
\label{appAddLemmas}
%---------------------------------------------%
%---------------------------------------------%

In this Appendix, we provide additional lemmas used in the proofs of our main results.

\subsection{Linear classification lemmas}

\begin{lemma}
Consider a linear classifier \(f(x) = \theta^{\trans} x + b\).
We have the upper bound
\begin{align*}
& \begin{aligned}
\frac{1}{n} \sum_{i = 1}^{n} & 
\left(\lhinge(\Psi f, z_{i}) - \lhinge(f, z_{i})\right)
%\\ &
\leq 
\epsilon \|\theta\|_{q} 
\frac{1}{n} \sum_{i = 1}^{n}  \lhzo(\Psi f, z_{i})
\end{aligned}
\end{align*}
and the lower bound
\begin{align*}
& \begin{aligned}
\frac{1}{n} \sum_{i = 1}^{n} &
\left(\lhinge(\Psi f, z_{i}) - \lhinge(f, z_{i})\right) 
%\\ &
\geq 
\epsilon \|\theta\|_{q} 
\frac{1}{n} \sum_{i = 1}^{n}  \lhzo(f, z_{i}).
\end{aligned}
\end{align*}
\label{lemmaLinearEmpiricalBound}
\end{lemma}

\begin{proof}[Proof of Lemma~\ref{lemmaLinearEmpiricalBound}]
Using Proposition~\ref{propSupLinear} and the fact that $\ell_h(f,z) = f(x)\ell_{h, 01}(f,z)$, we write the difference in losses as
\begin{align}
& \begin{aligned}
\lhinge(\Psi f, z)  - \lhinge(f, z) 
%\\ &
=
\left(1 - y (\theta^{\trans}x + b) + \epsilon \|\theta\|_{q} \right) 
\lhzo(\Psi f, z)
% \\ &\qquad 
-
\left(1 - y (\theta^{\trans}x + b)\right)
\lhzo(f, z).
\label{eqnLinearDiffRep}
\end{aligned}
\end{align}
We start by proving the upper bound.
Suppose \(\lhzo(f, z) = 1\).
Then $\ell_{h, 01} (\Psi f, z) \ge \ell_{h, 01}(f,z)$, so \(\lhzo(\Psi f, z) = 1\), as well, which means that
\begin{equation}
\label{EqnFirst}
\lhinge(\Psi f, z) - \lhinge(f, z) 
= 
\epsilon \|\theta\|_{q} \lhzo(\Psi f, z).
\end{equation}
If instead \(\lhzo(f, z) = 0\), we have
\[
\left(1 - y (\theta^{\trans}x + b) + \epsilon \|\theta\|_{q} \right) 
\leq 
0,
\]
so by equation~\eqref{eqnLinearDiffRep}, we have
\begin{align}
\label{EqnSecond}
&\begin{aligned}
\lhinge(\Psi f, z)  - \lhinge(f, z) 
%\\ &
&\leq 
\left(1 - y (\theta^{\trans}x + b) + \epsilon \|\theta\|_{q} \right) 
\lhzo(\Psi f, z) 
%\\ &\qquad 
-
\left(1 - y (\theta^{\trans}x + b)\right)
\lhzo(\Psi f, z)
\\ 
&=
\epsilon \|\theta\|_{q} \lhzo(\Psi f, z).
\end{aligned}
\end{align}
Averaging over all $i$ completes the upper bound.

The lower bound is very similar.
In detail, consider the case \(\lhzo(f, z) = 1\).
Once again, we have \(\lhzo(\Psi f, z) = 1\),
so
\[
\lhinge(\Psi f, z) - \lhinge(f, z) 
= 
\epsilon \|\theta\|_{q} \lhzo(f, z) .
\]
Next, suppose $\ell_{h, 01}(f,z) = 0$. Clearly, we then have
\[
\lhinge(\Psi f, z) - \lhinge(f, z) 
\geq 
0
=
\epsilon \|\theta\|_{q} \lhzo(f, z).
\]
Averaging over all $i$ completes the lower bound and the proof.
\end{proof}

%---------------------------------------------%
%---------------------------------------------%
\subsection{Neural network lemmas}

Here, we collect lemmas for neural networks.
We start with a bound on the difference between the empirical risks of \(Tf\) and \(f\).

%---------------------------------------------%
%---------------------------------------------%
\begin{lemma}
Let \(f(x) = A^{(d + 1)}s_{d}(\ldots s_{1}(A^{(1)} x))\) be a neural network with \(1\)-Lipschitz activation functions \(s_{j}\), applied elementwise. 
%Let \(\ell(f, z)\) be a \(1\)-Lipschitz function in \(f(x)\) that is decreasing in \(yf(x)\).
%Suppose $\ell(f,z)$ is decreasing in $yf(x)$ and $\ellbar$ is 1-Lipschitz in its first argument. 
Let $g_i(a) = \ellbar_{\text{xe}}(a,y_i)$.
%Let \(g_{i}(a) = \ell(a, y)\),
%and suppose 
%\(g_{i}(a) \leq g_{i}(b)\) implies 
%\(|g'_{i}(a)| \leq |g'_{i}(b)|\) for all $i$.
Then
\begin{align*}
&\begin{aligned}
\frac{1}{n}\sum_{i = 1}^{n} & \left( \ell_{\text{xe}}(T f, z_{i}) - \ell_{\text{xe}}(f, z_{i})\right)
%\\ 
&\leq
\epsilon 
\max_{j_2= 1, \ldots J_{2}} \left\|a_{j_2}^{(1)}\right\|_q
\prod_{j = 2}^{d + 1} \left\|A^{(j)}\right\|_{\infty}
% \\ & \qquad \times
\frac{1}{n} \sum_{i = 1}^{n} |g'_{i}(Tf(x_{i}, y_i))|.
\end{aligned}
\end{align*}
\label{lemmaNNPER2}
\end{lemma}

\begin{proof}[Proof of Lemma~\ref{lemmaNNPER2}]
We only need to prove the bound for a single summand, since we sum and then divide by \(n\).
%First, note that if \(g_{i}(a) \leq g_{i}(b)\), then we have the inequality
By Lemma~\ref{lemmaCrossEntropyLipschitz}, we have the inequality
\[
g_{i}(b) - g_{i}(a) \leq |g'_{i}(b)||b - a|.
\]
It follows that
%because \(g_{i}\) is \(|g'_{i}(b)|\)-Lipschitz on the interval between \(a\) and \(b\).
%Thus, we have
\begin{align*}
& \begin{aligned}
L_i 
&:=
\ell_{\text{xe}}(T f, z_i) - \ell_{\text{xe}}(f, z_i) = g_i(Tf(x_i, y_i)) - g_i(f(x_i)) \\ 
&\leq  
|g'_{i}(Tf(x_{i}, y_i))|
%\\ &\qquad \times
\biggr| \sum_{j_{d + 1} = 1}^{J_{d + 1}}
a^{(d + 1)}_{1, j_{d + 1}} s_{d}\left(\sum_{j_{d} = 1}^{J_{d}} a_{1, j_{d}}^{(d)} s_{d - 1}\left(
% \right. \right. \\ &\qquad \left. \left.
\ldots s_{1}\left(\left(a_{j_{2}}^{(1)}\right)^{\trans}x_{i} + w_{f}^{(j_{2: d + 1})} 
\right) \right) \right) 
\\ &\qquad 
- 
\sum_{j_{d + 1} = 1}^{J_{d + 1}}
a^{(d + 1)}_{1, j_{d + 1}}
s_{d}\left(\sum_{j_{d} = 1}^{J_{d}} a_{1, j_{d}}^{(d)} s_{d - 1}\left(\ldots 
%\right. \right. \\ &\qquad \left. \left.
s_{1}\left(\left(a_{j_{2}}^{(1)}\right)^{\trans}x_{i}  
\right) \right) \right)
\biggr|.
   \end{aligned}
\end{align*}

Now we need to peel off the layers of our neural networks. Applying Lemma~\ref{lemmaDiffPeeling} a total of \(d\) times, we have
\begin{align*}
& \begin{aligned}
L_i 
&\leq 
|g'_{i}(Tf(x_{i}, y_i))|
\left(\prod_{j = 2}^{d + 1} \left\|A^{(j)}\right\|_{\infty}\right) 
% \\ &\qquad 
\max_{j_{2}, \ldots, j_{d + 1}}
\biggr|
\left(a_{j_{2}}^{(1)}\right)^{\trans} x_{i} + w_{f}^{(j_{2 : d + 1})} - \left(a_{j_{2}}^{(1)}\right)^{\trans} x_{i}
\biggr| \\
&=
\left(\prod_{j = 2}^{d + 1} \left\|A^{(j)}\right\|_{\infty}\right)
\max_{j_{2}, \ldots, j_{d + 1}}
\biggr| 
w_{f}^{(j_{2 : d + 1})} 
\biggr| 
|g'_{i}(Tf(x_{i}))| \\ 
&=
\epsilon 
\max_{j_{2} = 1, \ldots J_{2}} \left\|a_{j_2}^{(1)}\right\|_q
\prod_{j = 2}^{d + 1} \left\|A^{(j)}\right\|_{\infty}
|g'_{i}(Tf(x_{i}))|,
   \end{aligned}
\end{align*}
where the last equality follows by the definition of the \(w_{f}^{(j_{2 : d + 1})}\).
Summing over $i$ and averaging proves the lemma.
\end{proof}
%---------------------------------------------%
%---------------------------------------------%

Next, we have two lemmas for peeling back the layers of a neural network.

%---------------------------------------------%
%---------------------------------------------%
\begin{lemma}
Let \(s: \reals \to \reals\) be a \(1\)-Lipschitz function applied elementwise to vectors.
Let \(a_{j'}^{\trans}\) denote the \(j'\)th row of \(A\),
and let \(b_{j}^{\trans}\) denote the \(j\)th row of \(B\). 
Let \(f_{j, j'}\) and \(f_{j, j'}'\) be functions from \(\reals^{m}\) to \(\reals^{K}\), for \(j = 1, \ldots, J\) and \(j' = 1, \ldots, J'\).
%\textcolor{red}{Need to define $f_j$ and $f_j'$.}
Then we have
\begin{align*}
&\begin{aligned}
\max_{j = 1, \ldots, J}  &
\biggr|\sum_{j' = 1}^{J'} b_{j, j'} 
s\left( a_{j'}^{\trans} f_{j, j'}(x)\right) 
- \sum_{j' = 1}^{J'} b_{j, j'} 
s\left( a_{j'}^{\trans} f_{j, j'}'(x)\right)\biggr| 
\\& 
\leq 
\|B\|_{\infty} \max_{j = 1, \ldots, J}  \max_{j' = 1, \ldots, J'} 
|a_{j'}^{\trans} f_{j, j'}(x) 
- a_{j'}^{\trans} f_{j, j'}'(x)|.
\end{aligned}
\end{align*}
\label{lemmaDiffPeeling}
\end{lemma}

\begin{proof}
Let \(L\) denote the left-hand side of the inequality.
Applying H\"{o}lder's inequality and using the fact that \(s\) is \(1\)-Lipschitz, we obtain
\begin{align*}
& \begin{aligned}
L 
&\leq 
\max_{j = 1, \ldots, J}
\max_{j' = 1, \ldots, J'}
\left(\sum_{j' = 1}^{J'} |b_{j, j'}|\right) 
\left| s\left(a_{j'}^{\trans} f_{j, j'}(x)\right)
- s\left(a_{j'}^{\trans} f_{j, j'}(x) \right)\right| \\ 
&\leq 
\|B\|_{\infty}
\max_{j = 1, \ldots, J}
\max_{j' = 1, \ldots, J'} 
\left|
a_{j'}^{\trans} f_{j, j'}(x)
-
a_{j'}^{\trans} f_{j, j'}(x)
\right|.
   \end{aligned}
\end{align*}
This establishes the lemma.
\end{proof}
%---------------------------------------------%
%---------------------------------------------%
The next lemma deals with the Rademacher complexity.
This is essentially the same as the lemmas of \cite{golowich2018}.

\begin{lemma}
Let $\{b_{j}\}$ be vectors such that \(\|b_{j}\|_{1} \leq \beta\), and let $\{a_j\}$ denote the rows of $A$. Let \(s\) be a \(1\)-Lipschitz activation function applied elementwise to vectors, such that \(s(0) = 0\). Let \(G\) be a convex, increasing, positive function.
Finally, let the \(f_{j, j'}: \reals^{m} \to \reals^{K}\) be functions. 
Then we have
\begin{align*}
& \begin{aligned}
\expect & \left[ \sup_{f \in \functions'} \max_{j = 1, \ldots, J} 
G\left(
\sum_{i = 1}^{n} \sigma_{i} 
\sum_{j' = 1}^{J'} 
b_{j, j'} s\left(a_{j'}^{\trans} f_{j, j'}(x_{i})\right)
\right) \right]
 \\
&\leq 
2 \expect \left[\sup_{f \in \functions'} \max_{j = 1, \ldots, J}  \max_{j' = 1,\ldots, J} 
% \right. \\ &\qquad \left.
G\left(\beta
\sum_{i = 1}^{n} \sigma_{i} a_{j'}^{\trans} f_{j, j'}(x_{i}) 
\right)\right].
   \end{aligned}
\end{align*}
\label{lemmaRademacherPeeling}
\end{lemma}

\begin{proof}
Let \(L\) denote the left-hand side in the statement of the lemma.
Using H\"{o}lder's inequality and the assumption that \(G\) is increasing, we have
\begin{align*}
& \begin{aligned}
L 
& :=
\expect \left[\sup_{f \in \functions'} \max_{j = 1, \ldots, J}
G\left(
\sum_{j' = 1}^{J'} b_{j, j'} 
\left(\sum_{i = 1}^{n} \sigma_{i} s\left(a_{j'}^{\trans} f_{j, j'}(x_{i})\right) \right)
\right)\right] \\
&\leq 
\expect \left[\sup_{f \in \functions'} \max_{j = 1, \ldots, J}
G\left(
\|b_{j}\|_{1} \max_{j' = 1, \ldots, J'}
\left|
\sum_{i = 1}^{n} \sigma_{i} s\left(a_{j'} f_{j, j'}(x_{i})\right)  \right|
\right)\right] \\ 
&\leq 
\expect \left[\sup_{f \in \functions'} \max_{j = 1, \ldots, J}  \max_{j' = 1, \ldots, J'}
%\right. \\ &\qquad \left.
G\left(
\beta \left|\sum_{i = 1}^{n} \sigma_{i} s\left(a_{j'}^{\trans} f_{j, j'}(x_{i}) \right)  \right|
\right)\right].
   \end{aligned}
\end{align*}
Now we perform a symmetrization step. Since \(G\) is positive and monotone, 
we have \(G(|x|) \leq G(x) + G(-x)\).
Combining this with the fact that \(\sigma_{i}\) and \(-\sigma_{i}\) have the same distribution, we obtain
\begin{align*}
& \begin{aligned}
L 
&\leq 
\expect\left[ \sup_{f \in \functions'} \max_{j = 1, \ldots, J} \max_{j' = 1, \ldots, J'}
%\right. \\ &\qquad \left.
G\left(
-\beta \sum_{i = 1}^{n} \sigma_{i} s\left(a_{j'}^{\trans} f_{j, j'}(x_{i})\right) 
\right)
%\right. \\ & \qquad \left. 
+
G\left(
\beta \sum_{i = 1}^{n} \sigma_{i} s\left(a_{j'}^{\trans} f_{j, j'}(x_{i})\right) 
\right)
\right] \\
&\leq 
\expect\left[ \sup_{f \in \functions'} \max_{j = 1, \ldots, J} \max_{j' = 1, \ldots, J'}
%\right. \\ &\qquad \left.
G\left(
-\beta \sum_{i = 1}^{n} \sigma_{i} s\left(a_{j'}^{\trans} f_{j, j'}(x_{i})\right) 
\right)\right]
\\ & \qquad 
+
\expect\left[ \sup_{f \in \functions'} \max_{j = 1, \ldots, J} \max_{j' = 1, \ldots, J'}
%\right. \\ &\qquad \left.
G\left(
\beta \sum_{i = 1}^{n} \sigma_{i} s\left(a_{j'}^{\trans} f_{j, j'}(x_{i})\right) 
\right)\right] \\
&=
2\expect \left[\sup_{f \in \functions'} \max_{j = 1, \ldots, J} \max_{j' = 1, \ldots, J'}
%\right. \\ &\qquad \left.
G\left(
\beta \sum_{i = 1}^{n} \sigma_{i} s\left(a_{j'}^{\trans} f_{j, j'}(x_{i})\right) 
\right)\right].
   \end{aligned}
\end{align*}
Finally, we apply Lemma~\ref{lemmaTalagrandsLemma1991} to obtain
\begin{align*}
& \begin{aligned}
L 
&\leq 
2\expect\left[\sup_{f \in \functions'} \max_{j = 1, \ldots, J} \max_{j' = 1, \ldots, J'}
% \right. \\ &\qquad \left.
G\left(
\beta \sum_{i = 1}^{n} \sigma_{i} a_{j'}^{\trans} f_{j, j'}(x_{i})
\right)\right].
   \end{aligned}
\end{align*}
This completes the proof.
\end{proof}

%---------------------------------------------%
%---------------------------------------------%

%---------------------------------------------%
%---------------------------------------------%
\section{Multiclass proofs}
\label{AppMulticlass}

In this Appendix, we provide proofs of the results on multiclass classification stated in Section~\ref{secMulticlass}.

\begin{proof}[Proof of Lemma~\ref{lemmaMonotonic}]
We write $a = (a(1), \dots, a(K))$ and $b = (b(1), \dots, b(K))$, and define the vectors \(a_{0}, \ldots, a_{K}\) by
\begin{align*}
& \begin{aligned}
a_{0}  &= (a(1), \ldots, a(K)), \\
a_{1}  &= (b(1), a(2), \ldots, a(K)), \\
 &\; \; \vdots  \\
a_{K}  &= (b(1), \ldots, b(K)).
   \end{aligned}
\end{align*}
Note that for each $1 \le k \le K$, the vectors $a_k$ and $a_{k-1}$ differ in only one position. Since $a_k(k) = b(k)$ and $a_{k-1}(k) = a(k)$, and $a(k) y^{(k)} \ge b(k) y^{(k)}$, we conclude from the fact that $\ell$ is coordinatewise decreasing that
%Then, we have
\begin{align*}
& \begin{aligned}
0 &\leq \ell(a_{k}, y) - \ell(a_{k - 1}, y).
   \end{aligned}
\end{align*}
%since \(\ell\) is coordinate-wise decreasing in \(y \circ f(x)\).
Summing over \(k\), we then obtain
\begin{align*}
& \begin{aligned}
0 
&\leq 
\sum_{k = 1}^{K} \left(\ell(a_{k}, y) - \ell(a_{k - 1}, y)\right)
=
\ell(b, y) - \ell(a, y).
   \end{aligned}
\end{align*}
Rearranging gives the desired result.
\end{proof}
%---------------------------------------------%
%---------------------------------------------%
\begin{proof}[Proof of Proposition~\ref{propPsiMonotone}]
We examine each coordinate individually.
Suppose \(y^{(k)} = +1\). By the definition of the sup transform, we need to show that
\begin{align*}
& \begin{aligned}
f_{k}(x + w) 
&\geq (-1) \sup_{w' \in B(\epsilon)} (-1) f_{k}(x + w').
   \end{aligned}
\end{align*}
Equivalently,
\[
(-1)f_{k}(x + w)  \leq \sup_{w' \in B(\epsilon)} (-1) f_{k}(x + w'),
\]
which is obviously true.

Similarly, for \(y^{(k)} = -1\), we need to show that
\begin{align*}
& \begin{aligned}
f_{k}(x + w) 
&\leq \sup_{w' \in B(\epsilon)} f_{k}(x + w'),
   \end{aligned}
\end{align*}
which is clearly also true. Thus, we conclude that \(f(x + w) \preceq_{y} \Psi f(x, y)\),
and the rest of the proposition follows from Lemma~\ref{lemmaMonotonic}.
\end{proof}

%---------------------------------------------%
%---------------------------------------------%

\begin{proof}[Proof of Lemma~\ref{lemmaMulticlassLinearERC}]
The calculation is essentially the same as in the binary classification case.
We have
\begin{align*}
& \begin{aligned}
n \emprad_{n}(\Pi_{1}(\Psi \functions_{\linear})) 
&=
\expect _{\sigma} \left[\sup_{f \in \functions_{\linear}, k \in \{1, \ldots, K\}}
\sum_{i = 1}^{n}
\sigma_{i}\left(\theta_{k}^{\trans} x_i + b_{k} - y_{i}^{(k)} \epsilon \|\theta_{k}\|_{q}\right)\right] \\
&\leq 
\expect _{\sigma} \left[\sup_{f \in \functions_{\linear}, k \in \{1, \ldots, K\}} 
\sum_{i = 1}^{n} \sigma_{i} \left(\theta_{k}^{\trans} x + b_{k}\right)\right]
+
\epsilon
\expect _{\sigma} \left[\sup _{f \in \functions_{\linear}, k \in \{1, \ldots, K\}}
\|\theta_{k}\|_{q} \sum_{i = 1}^{n} \sigma_{i}\right] \\ 
&=
n \emprad_{n}(\Pi_{1}(\functions_{\linear}))
+
\epsilon \expect _{\sigma} \left[\sup _{f \in \functions_{\linear}, k \in \{1, \ldots, K\}}
\|\theta_{k}\|_{q} \sum_{i = 1}^{n} \sigma_{i}\right].
   \end{aligned}
\end{align*}
From Lemma~\ref{lemmaMulticlassLinearRademacher}, we can bound the first term; so it only remains to bound the second.
Note that when \(\sum_{i = 1}^{n} \sigma_{i} \geq 0\), the integrand is maximized when \(\|\theta_{k}\|_{q} = M_{q}\).
On the other hand, if \(\sum_{i = 1}^{n} \sigma_{i} \leq 0\), the integrand is maximized for \(\|\theta_{k}\|_{q} = 0\).
Thus, as in the proof of Lemma~\ref{lemmaLinearERC}, we have 
\begin{align*}
& \begin{aligned}
\epsilon \expect _{\sigma} 
\left[\sup _{f \in \functions_{\linear}}
\|\theta_{k}\|_{q} \sum_{i = 1}^{n} \sigma_{i}\right]
&=
\epsilon M_{q} \expect _{\sigma}
\left[\sup_{f \in \functions_{\linear}}
 \sum_{i = 1}^{n} \sigma_{i} 
\ind \left\{\sum_{i = 1}^{n} \sigma_{i} > 0\right\}\right] \\
&=
\frac{\epsilon M_{q}}{2} 
\expect _{\sigma} 
\left[\left| \sum_{i = 1}^{n} \sigma_{i}\right|\right] \\ 
&=
\frac{\epsilon M_{q}}{2} 
\sqrt{\expect _{\sigma} 
\left[\left( \sum_{i = 1}^{n} \sigma_{i}\right)^{2}\right]
} \\ 
&=
\frac{\epsilon M_{q} \sqrt{n}}{2}.
   \end{aligned}
\end{align*}
Putting everything together completes the proof.
\end{proof}

%---------------------------------------------%
%---------------------------------------------%
\section{Regression proofs}
\label{AppRegression}

In this Appendix, we provide proofs of the results appearing in Section~\ref{secRegression}.

\subsection{Transform proofs}

\begin{proof}[Proof of Proposition~\ref{propRegressionLossMinMax}]
From equation~\eqref{eqnLRTRUNC}, we can write
\begin{align*}
& \begin{aligned}
\sup_{w \in B(\epsilon)} \lrtrunc(f, z + w) 
&=
\sup_{w \in B(\epsilon)} \lrtruncpn(f, f, z + w) \\
&=
\max\left\{
\sup_{w \in B(\epsilon)} 
\min\left\{(f(x + w) - y)_{+}^{r}, B^{r}\right\}, 
\right. \\ & \left. \hspace{43pt}
\sup_{w \in B(\epsilon)}
\min\left\{(f(x + w) - y)_{-}^{r}, B^{r}\right\}
\right\} \\
&=
\max\left\{
\min\left\{\sup_{w \in B(\epsilon)} (f(x + w) - y)_{+}^{r}, B^{r}\right\},
\right. \\ & \left. \hspace{43pt}
\min\left\{\sup_{w \in B(\epsilon)} (f(x + w) - y)_{-}^{r}, B^{r}\right\}
\right\} \\
&=
\max\left\{
\min\left\{(\Psi_+ f(x + w) - y)_{+}^{r}, B^{r}\right\},
\min\left\{(\Psi_- f(x + w) - y)_{-}^{r}, B^{r}\right\}
\right\} \\
& = \ell_{r,B}^{\pm}(\Psi_+ f, \Psi_- f, z),
   \end{aligned}
\end{align*}
which completes the proof.
\end{proof}

%---------------------------------------------%
%---------------------------------------------%
\begin{proof}[Proof of Lemma~\ref{lemmaRademacherRegressionLoss}]
To obtain the first inequality, we apply Lemma~\ref{lemmaRadMax} once:
\begin{align*}
& \begin{aligned}
\emprad_{n}(\lrtrunc \circ 
&(\functions, \mathcal{G})) \\
&=
\frac{1}{n}\expect _{\sigma} \left[\sup_{f \in \functions, g \in \mathcal{G}} 
\sum_{i = 1}^{n}
\sigma_{i}
\max\left\{ 
\min\left\{( f(x) - y)_{+}^{r}, B\right\}, 
\min\left\{( g(x) - y)_{-}^{r}, B\right\}
\right\}\right] \\
&\leq
\emprad_{n}(\lrplus \circ \functions)
+
\emprad_{n}(\lrminus \circ \mathcal{G}).
   \end{aligned}
\end{align*}
To obtain the second inequality in the statement of the lemma, we note that \(\lrplus\) and \(\lrminus\) are \(rB^{r - 1}\)-Lipschitz as a function of $(f(x) - y)$.
Thus, by Lemma~\ref{lemmaTalagrandsLemma1991}, we have
\begin{align*}
& \begin{aligned}
n\emprad_{n}(\lrplus \circ \functions) 
&\leq
rB^{r - 1} 
\expect _{\sigma} \left[\sup _{f \in \functions} 
\sum_{i = 1}^{n} \sigma_{i} (f(x_{i}) - y_{i}) \right] \\
&=
rB^{r - 1}
\expect _{\sigma} \left[\sup _{f \in \functions} 
\sum_{i = 1}^{n} \sigma_{i} f(x_{i}) \right] \\
&=
rB^{r - 1}
\emprad_{n}(\functions).
   \end{aligned}
\end{align*}
The bound for \(\emprad_{n}(\lrminus \circ \mathcal{G})\) holds analogously, completing the proof.
\end{proof}

%---------------------------------------------%
%---------------------------------------------%
\begin{comment}
\begin{proof}[Proof of Lemma~\ref{lemmaLinearPosNeg}]
We can prove this proposition by appealing to the supremum transform. Specifically, we observe
\[
\ptrans f(x) = \Psi f(x, -1) = \theta^{\trans} x + b + \epsilon \|\theta\|_{q}
\]
and 
\[
\ntrans f(x) = \Psi f(x, +1) = \theta^{\trans}x + b - \epsilon \|\theta\|_{q}.
\]
This completes the proof.
\end{proof}
\end{comment}

%---------------------------------------------%
%---------------------------------------------%
\begin{proof}[Proof of Proposition~\ref{propTreePosNegLoss}]
It suffices to show that 
\(\ptrans f(x) \leq T_{+}f(x)\) and \(\ntrans f(x) \geq T_{-}f(x)\).
This can be shown in a straightforward manner using the supremum and tree transforms.
In particular, we have
\[
\ptrans f(x) = \Psi f(x, -1) \leq Tf(x, -1) = T_{+}f(x)
\]
and
\[
\ntrans f(x) = \Psi f(x, +1) \geq Tf(x, +1) = T_{-}f(x).
\]
This completes the proof.
\end{proof}

%---------------------------------------------%
%---------------------------------------------%
\subsection{Rademacher complexity proofs}

\begin{proof}[Proof of Corollary~\ref{corRademacherLinearPostNeg}]
We start with the Rademacher complexity of \(\ptrans\functions_{\linear}\).
We have
\[
\emprad_{n}(\ptrans \functions_{\linear})
=
\frac{1}{n} \expect _{\sigma}\left[ \sup_{f \in \functions}
\sum_{i = 1}^{n} \sigma_{i} \ptrans f(x_{i}) \right]
=
\frac{1}{n} \expect _{\sigma} \left[\sup_{f \in \functions}
\sum_{i = 1}^{n} \sigma_{i} \Psi f(x_{i}, -1) \right]
=
\emprad_{n}(\Psi \functions_{\linear}),
\]
where the final empirical Rademacher complexity is computed with respect to the observations $\{(x_i, -1)\}$. By Lemma~\ref{lemmaLinearERC}, this value is upper-bounded by $\frac{M_2 R}{\sqrt{n}} + \frac{\epsilon M_q}{2\sqrt{n}}$.

Similarly, we have
\[
\emprad_{n}(\ntrans \functions_{\linear})
=
\frac{1}{n} \expect _{\sigma} \left[\sup_{f \in \functions}
\sum_{i = 1}^{n} \sigma_{i} \ntrans f(x_{i}) \right]
=
\frac{1}{n} \expect _{\sigma} \left[\sup_{f \in \functions}
\sum_{i = 1}^{n} \sigma_{i} \Psi f(x_{i}, +1) \right]
=
\emprad_{n}(\Psi \functions_{\linear}),
\]
where the final empirical Rademacher complexity is computed with respect to the observations $\{(x_i, +1)\}$. This completes the proof.
\end{proof}
%---------------------------------------------%
%---------------------------------------------%
\begin{proof}[Proof of Corollary~\ref{corNNTreePosNegRadBound}]
Again, our goal is to instantiate the transforms that we have from classification to obtain an equivalence between the empirical Rademacher complexities of \(T_{+}\functions_{\nets}\) and \(T_{-}\functions_{\nets}\) and a suitable Rademacher complexity of \(T\functions_{\nets}\).
Here, we see that
\[
\emprad_{n}(T_{+}\functions_{\nets})
=
\frac{1}{n} \expect _{\sigma} \left[\sup_{f \in \functions}
\sum_{i = 1}^{n} \sigma_{i} T_{+}f(x_{i})\right]
=
\frac{1}{n} \expect _{\sigma} \left[\sup_{f \in \functions}
\sum_{i = 1}^{n} \sigma_{i} Tf(x_{i}, -1)\right]
=
\emprad_{n}(T \functions_{\nets}) 
\]
and 
\[
\emprad_{n}(T_{-}\functions_{\nets})
=
\frac{1}{n} \expect _{\sigma}\left[ \sup_{f \in \functions}
\sum_{i = 1}^{n} \sigma_{i} T_{-}f(x_{i})\right]
=
\frac{1}{n} \expect _{\sigma} \left[\sup_{f \in \functions}
\sum_{i = 1}^{n} \sigma_{i} Tf(x_{i}, +1)\right]
=
\emprad_{n}(T \functions_{\nets}) ,
\]
where in the empirical Rademacher complexity of \(T \functions_{\nets}\), the \(y_{i}\)'s are taken to be \(-1\) in the first equation and \(+1\) in the second.
The bound on the empirical Rademacher complexity of \(T \functions_{\nets}\) from Lemma~\ref{lemmaNNRad2} then completes the proof.

\end{proof}
%---------------------------------------------%
%---------------------------------------------%
\subsection{Risk bound proofs}

\begin{proof}[Proof of Corollary~\ref{corLinRegRiskBound}]
Our goal is to use the standard generalization bound of Lemma~\ref{thmMainGeneralization}.
To do this, we need to rescale the loss to take values in \([0, 1]\). Since the maximum loss is \(B^{r}\), we have
\begin{align*}
& \begin{aligned}
\frac{1}{B^{r}} \expect \lrtruncpn(\ptrans f, \ntrans f, z) 
&\leq 
\frac{1}{B^{r}} \cdot \frac{1}{n} 
\sum_{i = 1}^{n} \lrtruncpn(\ptrans f, \ntrans f, z_{i})
\\ & \qquad 
+
\frac{2}{B^{r}} 
\emprad_{n}(\lrtruncpn \circ 
(\ptrans \functions_{\linear}, \ntrans \functions_{\linear}))
+
3 \sqrt{\frac{\log \frac{2}{\delta}}{2n}},
   \end{aligned}
\end{align*}
with probabiilty at least $1-\delta$, by Lemma~\ref{thmMainGeneralization}, so
\begin{align*}
& \begin{aligned}
\expect \lrtruncpn(\ptrans f, \ntrans f, z) 
&\leq 
\frac{1}{n} 
\sum_{i = 1}^{n} \lrtruncpn(f, z_{i})
+
2
\emprad_{n}(\lrtruncpn \circ 
(\ptrans \functions_{\linear}, \ntrans  \functions_{\linear}))
+
3 B^{r} \sqrt{\frac{\log \frac{2}{\delta}}{2n}}.
   \end{aligned}
\end{align*}
Applying Lemma~\ref{lemmaRademacherRegressionLoss} and Corollary~\ref{corRademacherLinearPostNeg} then gives
\begin{align*}
& \begin{aligned}
\expect \lrtruncpn(\ptrans f, \ntrans f, z) 
&\leq 
\frac{1}{n} 
\sum_{i = 1}^{n} \lrtruncpn(f, z_{i})
+
2r B^{r - 1}
\emprad_{n}(\ptrans \functions_{\linear})
+
2r B^{r - 1}
\emprad_{n}(\ntrans \functions_{\linear})
+
3 B^{r} \sqrt{\frac{\log \frac{2}{\delta}}{2n}} \\
&\leq 
\frac{1}{n} 
\sum_{i = 1}^{n} \lrtruncpn(f, z_{i})
+
2r B^{r - 1}
\left(
\frac{M_{2}R}{\sqrt{n}} + \frac{\epsilon M_{q}}{2\sqrt{n}}
\right)
+
2r B^{r - 1}
\left(
\frac{M_{2}R}{\sqrt{n}} + \frac{\epsilon M_{q}}{2\sqrt{n}}
\right)
\\ &\qquad 
+
3 B^{r} \sqrt{\frac{\log \frac{2}{\delta}}{2n}},
   \end{aligned}
\end{align*}
and this completes the proof.
\end{proof}

%---------------------------------------------%
%---------------------------------------------%
\begin{proof}[Proof of Corollary~\ref{corNNRegressionRiskBound}]
This proof parallels the proof of Corollary~\ref{corLinRegRiskBound}.
Applying Lemma~\ref{thmMainGeneralization} to the rescaled loss function, we obtain
\begin{align*}
& \begin{aligned}
\frac{1}{B^{r}} \expect \lrtruncpn(T_{+}f, T_{-}f, z) 
&\leq 
\frac{1}{B^{r}} \cdot \frac{1}{n} 
\sum_{i = 1}^{n} \lrtruncpn(T_{+}f, T_{-}f, z_{i})
+
\frac{2}{B^{r}} 
\emprad_{n}(\lrtruncpn \circ 
(T_{+}\functions_{\nets}, T_{-} \functions_{\nets}))
\\&\qquad 
+
3 \sqrt{\frac{\log \frac{2}{\delta}}{2n}},
   \end{aligned}
\end{align*}
with probability at least $1-\delta$. Thus,
\begin{align*}
& \begin{aligned}
\expect \lrtruncpn(T_{+}f, T_{-}f, z) 
&\leq 
\frac{1}{n} 
\sum_{i = 1}^{n} \lrtruncpn(T_{+}f, T_{-}f, z_{i})
+
2
\emprad_{n}(\lrtruncpn \circ 
(T_{+}\functions_{\nets}, T_{-} \functions_{\nets}))
+
3 B^{r} \sqrt{\frac{\log \frac{2}{\delta}}{2n}}.
   \end{aligned}
\end{align*}
Applying Lemma~\ref{lemmaRademacherRegressionLoss} and Corollary~\ref{corNNTreePosNegRadBound} then gives
\begin{align*}
& \begin{aligned}
\expect \lrtruncpn(T_{+}f, T_{-}f, z) 
&\leq 
\frac{1}{n} 
\sum_{i = 1}^{n} \lrtruncpn(T_{+}f, T_{-}f, z_{i})
+
2 rB^{r - 1}
\emprad_{n}(T_{+}\functions_{\nets})
+
2 rB^{r - 1}
\emprad_{n}(T_{-}\functions_{\nets})
\\&\qquad 
+
3 B^{r} \sqrt{\frac{\log \frac{2}{\delta}}{2n}} \\ 
&\leq 
\frac{1}{n} 
\sum_{i = 1}^{n} \lrtruncpn(T_{+}f, T_{-}f, z_{i})
+
2 rB^{r - 1}
\left(
\alpha\left( \frac{\alpha_{1, F}}{\alpha_{1}}R + \frac{\alpha_{1, q}}{\alpha_{1}} \epsilon\right) 
\cdot \frac{\sqrt{2d \log 2} + 1}{\sqrt{n}}
\right)
\\&\qquad 
+
2 rB^{r - 1}
\left(
\alpha\left( \frac{\alpha_{1, F}}{\alpha_{1}}R + \frac{\alpha_{1, q}}{\alpha_{1}} \epsilon\right) 
\cdot \frac{\sqrt{2d \log 2} + 1}{\sqrt{n}}
\right)
+
3 B^{r} \sqrt{\frac{\log \frac{2}{\delta}}{2n}},
   \end{aligned}
\end{align*}
completing the proof.
\end{proof}

%---------------------------------------------%
%---------------------------------------------%
\section{Auxiliary lemmas}
\label{appAuxLemmas}
In this section, we collect auxiliary results.

\subsection{Binary classification}

We start with a standard generalization bound \citep{mohri2012}.

\begin{lemma}
Let \(\functions\) be a class of functions. Let \(\ell\) be a loss function that takes values in \([0, 1]\) and is \(1\)-Lipschitz in \(f(x)\). With probability at least $1-\delta$, we have
\begin{align*}
&\begin{aligned}
\expect_{\dist}\left[\ell(f, z)\right]
&\leq 
\frac{1}{n} \sum_{i = 1}^{n}  \ell(f, Z_{i})
+
2 \emprad_{n}(\functions)
+
3\sqrt{\frac{\log \frac{2}{\delta}}{2n}}.
\end{aligned}
\end{align*}
\label{thmMainGeneralization}
\end{lemma}

Next, we derive a result concerning the Lipschitz continuity of the cross-entropy loss composed with a softmax activation function.

\begin{lemma}
Define the function \(g_y(a) = \ellbar_{\text{xe}}(a, y)\).
The derivative is given by
\[
g_y'(a)
=
\begin{cases}
\frac{-1}{\exp(a) + 1}, & y = +1 \\
\frac{\exp(a)}{\exp(a) + 1}, & y = -1.
\end{cases} 
\] 
In particular, the function \(g_y'(a)\) is monotonic and bounded in magnitude by \(1\), and
\begin{equation}
\label{EqnG}
g_y(a) - g_y(b) \le |g_y'(b)| \cdot |b-a|,
\end{equation}
for all $a,b \in \real$.
\label{lemmaCrossEntropyLipschitz}
\end{lemma}

\begin{proof}
Substituting the expression for $\delta(a)$ into the loss $\ellbar_{\text{xe}}$, we have
\[
g_y(a) 
=
\begin{cases}
- \log\left(\frac{\exp(a)}{\exp(a) + 1}\right), & y = +1 \\
- \log\left(\frac{1}{\exp(a) + 1}\right), & y = -1.
\end{cases} 
\] 
Thus, $g_y$ is monotonically decreasing when $y = +1$, and monotonically increasing when $y  = -1$, yielding equation~\eqref{EqnG}. Differentiating yields the desired expression for $g_y'$, and it is easy to see that the function is always monotonic and bounded by 1, as claimed.
\end{proof}

%---------------------------------------------%
%---------------------------------------------%

We also derive a bound on the empirical Rademacher complexity of a linear classifier.

\begin{lemma}
Suppose \(\|x_{i}\|_{2} \leq R\) for all \(i\).
Let \(\functions_{\linear}\) be a class of linear functions of the form \(f(x) = \theta^{\trans} x + b\).
If \(\|\theta\|_{2} \leq M_{2}\) for all \(f\) in \(\functions_{\linear}\), 
then the empirical Rademacher complexity satisfies
\[
\emprad_{n}(\functions_{\linear})
\leq 
\frac{M_{2} R}{\sqrt{n}}.
\]
\label{lemmaLinearRademacher}
\end{lemma}

\begin{proof}
Using the Cauchy-Schwarz inequality and Jensen's inequality, we obtain
\begin{align*}
& \begin{aligned}
\emprad(\functions)
&=
\frac{1}{n}
\expect_{\sigma}\left[
\sup_{\theta \in \functions} \theta^{\trans} \left(\sum_{i = 1}^{n} \sigma_{i}  x_{i}\right)
\right] \leq 
\frac{1}{n} \expect_{\sigma}\left[
\sup_{\theta \in \functions} \|\theta\|_{2} \left\|\sum_{i = 1}^{n} \sigma_{i}  x_{i}\right\|_{2} 
\right] \\ 
&\leq 
\frac{M_{2}}{n} \expect_{\sigma} 
\left\|\sum_{i = 1}^{n} \sigma_{i}  x_{i}\right\|_{2} \leq 
\frac{M_{2}}{n} \left( \expect_{\sigma}
\left[
\left\|\sum_{i = 1}^{n} \sigma_{i}  x_{i}\right\|_{2}^{2}
\right]\right)^{\frac{1}{2}} .
\end{aligned}
\end{align*}
Further note that 
\begin{align*}
& \begin{aligned}
\expect_{\sigma}
\left\|\sum_{i = 1}^{n} \sigma_{i}  x_{i}\right\|_{2}^{2}
&=
\expect_{\sigma}
\left[
\sum_{i,j = 1}^{n} \sigma_{i} \sigma_{j}  x_{i}^{\trans}  x_{j}
\right]  
&=
\expect_{\sigma}\left[
\sum_{i = 1}^{n} \|x_{i}\|_{2}^{2}
\right] 
&\leq 
nR^{2}.
\end{aligned}
\end{align*}
Putting everything together gives
\[
\emprad(\functions)
\leq 
\frac{M_{2}}{n} \sqrt{nR^{2}} 
=
\frac{M_{2} R}{\sqrt{n}},
\]
as desired.
\end{proof}

%---------------------------------------------%
%---------------------------------------------%
We also provide a bound on the cumulant generating function of a centered random variable and the resulting bounded differences inequality, which is given as Theorem~6.2 of \cite{boucheron2013}.

\begin{lemma}
Let \(f: \xspace^{n} \to \reals\) be a function satisfying the bounded differences assumption
\[
f(x_{1}, \ldots, x_{i}, \ldots, x_{n})
-
f(x_{1}, \ldots, x_{i}', \ldots, x_{n})
\leq 
c_{i}
\]
for all \(x_{i}\) and \(x_{i}'\) in \(\xspace\).
Define the variance factor 
\[
v
:=
\frac{1}{4} \sum_{i = 1}^{n} c_{i}^{2}.
\]
Let \(Z = f(x_{1}, \ldots, x_{n})\), where the \(x_{i}\)'s are independent random variables.
Then
\[
\log \expect \exp(\lambda(Z - \expect[Z]))
\leq 
\frac{\lambda^{2} v}{2}
\]
and
\[
\prob\left\{Z - \expect Z > t\right\}
\leq 
e^{-\frac{t^{2}}{2v}}.
\]
\label{lemmaBddDiff}
\end{lemma}

%---------------------------------------------%
%---------------------------------------------%
Finally, we provide Talagrand's contraction lemma. The term ``contraction'' refers to a \(1\)-Lipschitz function, although one can  easily extend the result to any \(L\)-Lipschitz function. The version stated here appears as equation~(4.20) in \cite{ledoux1991}. A similar statement appears as Proposition~4 of \cite{ledoux1989}.

\begin{lemma}
Let \(G\) be a convex, increasing function.
Let \(\phi_{i}: \reals \to \reals\) be \(1\)-Lipschitz functions such that \(\phi_{i}(0) = 0\).
Let \(T\) be a compact subset of \(\reals^{n}\).
Then
\[
\expect G \left( \sup _{t \in T} \sum_{i = 1}^{n} \sigma_{i} \phi_{i}(t_{i})\right)
\leq 
\expect G \left(\sup _{t \in T} \sum_{i = 1}^{n} \sigma_{i} t_{i}\right).
\]
\label{lemmaTalagrandsLemma1991}
\end{lemma}

%---------------------------------------------%
%---------------------------------------------%
\subsection{Multiclass classification}
\label{AppMultClass}

Now, we need a bound for the multiclass risk. 
The following bound is a slight adaptation of Theorem~2 of \cite{kuznetsov2015}, applied in the case of empirical Rademacher complexities. Similar adaptations are made in \cite{mohri2012}.
%Before we state the theorem, we define
%\[
%\Pi_{1}(\functions) 
%= 
%\left\{
%(x, y) \mapsto f_{\tilde{y}}(x): y \in \yspace, f \in \functions
%\right\}.
%\]

\begin{lemma}
Let \(\lmargin\) be the margin loss. Suppose that there are \(K\) classes.
Then, with probability \(1 - \delta\), for any \(f \in \functions\), we have
\begin{align*}
& \begin{aligned}
\expect _{\dist} \lmargin(f, z)
&\leq 
\frac{1}{n} \sum_{i = 1}^{n} \ell(f, z_{i})
+
\frac{8K}{\rho} \emprad_{n}(\Pi_{1} (\functions))
+
3 \sqrt{\frac{\log \frac{2}{\delta}}{2n}}.
   \end{aligned}
\end{align*}
\label{lemmaMulticlassGeneralization}
\end{lemma}

Next, we have a simple bound on the multiclass Rademacher complexity of linear classifiers. This reduces to the usual Rademacher complexity for linear classifiers in the binary case, and the upper bound is the same. The proof is standard and can be found in \cite{mohri2012} as Proposition~8.1.
First, let \(\functions_{\linear}\) consist of linear classifiers such that an element \(f\) of \(\functions_{\linear}\) can be written as \(f(x) = \Theta x + b\).

\begin{lemma}
Consider the class of linear functions
\[
\functions_{\linear} 
=
\{f(x) = \Theta x + b: \|\Theta\|_{2, \infty} \leq M_{2} \}.
\]
Let \(R\) be such that \(\|x_{i}\|_{2} \leq R\) for all \(i\).
Then, we have the bound
\[
\emprad_{n}(\functions_{\linear})
=
\frac{M_{2} R}{\sqrt{n}}.
\]
\label{lemmaMulticlassLinearRademacher}
\end{lemma}

%---------------------------------------------%
%---------------------------------------------%
\subsection{Regression}

We also need a lemma on dealing with a maximum within a Rademacher complexity. This is a standard result \citep{mohri2012}.

\begin{lemma}[Lemma~8.1 of \citealt{mohri2012}]
Let \(\functions_{1}, \ldots, \functions_{l}\) be \(l\) hypothesis sets in \(\reals^{\xspace}\), \(l \geq 1\), and let 
\[
\mathcal{G} 
= 
\{\max\{h_{1}, \ldots, h_{l}\}: h_{i} \in \functions_{i}, i \in [1, l]\}.
\]
Then, for any sample \(S\) of size \(n\), the empirical Rademacher complexity of \(\mathcal{G}\) can be upper-bounded as follows:
\[
\emprad_{n}(\mathcal{G})
\leq 
\sum_{j = 1}^{l} \emprad_{n}(\functions_{j}).
\] 
\label{lemmaRadMax}
\end{lemma}

%---------------------------------------------%
%---------------------------------------------%
\section{Comparison of adversarial loss bounds}
\label{appIncomparability}

In this Appendix, we examine the difference between our upper bound on the adversarial loss for multiclass classification, denoted by \(\lmargin(T f, z) = \phi_{\rho}(m_{Tf}(z))\), and the loss proposed by \cite{yin2018}:
\begin{equation}
\hat{\ell}(f, z)
:=
\phi_{\rho}\left(m_{f}(z)
- \frac{\epsilon}{2} 
\max_{k \in [K], \; z \in \{+1, -1\}} \max_{ P \succeq 0, \; \diag(P) \leq 1} 
\left\langle z Q\left(A^{(2)}_{k}, A^{(1)}\right), P \right\rangle
\right),
\label{eqnYinLoss}
\end{equation}
where 
\begin{equation}
Q(v, W)
:=
\left[
\begin{array}{ccc}
0 & 0 & \ind ^{\trans} W^{\trans} \diag(v) \\
0 & 0 & W^{\trans} \diag(v) \\
\diag(v)^{\trans} W \ind & \diag(v)^{\trans} W & 0
\end{array}
\right].
\label{eqnQmatrix}
\end{equation}
Note that the analysis of \cite{yin2018} is derived only for a single-layer neural network, which we denote by $f(x) = A^{(2)} s(A^{(1)} x)$. Furthermore, adversarial perturbations are taken over the $\ell_\infty$-ball.

For simplicity, we define the semidefinite program term 
\[
\sdp(Q) 
:=
\max_{z \in \{+1, -1\}} \max_{P \succeq 0, \; \diag(P) \leq 1}
\left\langle z Q, P \right\rangle.
\]
The following proposition shows that the losses $\ell_\rho$ and $\hat{\ell}$ are incomparable in general, meaning that one loss does not uniformly dominate the other:

\begin{proposition}
There exists a neural network \(f\) and a data point \(z\) such that 
\(\lmargin(Tf, z) < \hat{\ell}(f, z)\).
Additionally, there exists a neural network \(f'\) and a data point \(z'\) such that 
\(\hat{\ell}(f', z') < \lmargin(Tf', z')\).
\label{propIncomparable}
\end{proposition}

We prove this proposition by considering the following network:
\begin{equation}
f(x)
=
\left[ 
\begin{array}{cc}
1 & 1 \\
-a & -a \\
-b & -b
\end{array}
\right]
s\left(
\left[ 
\begin{array}{ccc}
1  &  2 & 3 \\
10 & 20 & 30  
\end{array}
\right]
x
\right),
\label{eqnfExample}
\end{equation}
where \(s\) is the ReLU activation function, and $0 < a < b$ are constants that will be defined later.
%Note that though all of the rows of \(W^{(2)}\) are collinear, this is just to simplify the computation of the \(Q\) matrices, and we do not believe that the result depends crucially on this.

To calculate the loss with respect to \(\hat{\ell}\), we use the following computational lemma:
\begin{lemma}
Let \(f(x) = A^{(2)} s( A^{(1)} x)\) be as in equation~\eqref{eqnfExample}.
Then the SDP solutions are given by
\begin{align*}
\sdp\left(Q\left(A^{(2)}_{1}, A^{(1)} \right) \right) 
&=
264, \\
\sdp\left(Q\left(A^{(2)}_{2}, A^{(1)} \right) \right) & = 264a, \\
\sdp\left(Q\left(A^{(2)}_{3}, A^{(1)} \right) \right) & = 264b.
\end{align*}
\label{lemmaSDPSolution}
\end{lemma}

\begin{proof}
First, we compute the matrix \(Q(A^{(2)}_{k}, A^{(1)})\) for \(k = 1, 2, 3\).
We have 
\[
Q\left(A^{(2)}_{1}, A^{(1)}\right)
=
\left[
\begin{array}{cccccc}
0 & 0 & 0 & 0 & 6 & 60 \\
0 & 0 & 0 & 0 & 1 & 10 \\
0 & 0 & 0 & 0 & 2 & 20 \\
0 & 0 & 0 & 0 & 3 & 30 \\
6 & 1 & 2 & 3 & 0 & 0 \\
60 & 10 & 20 & 30 & 0 & 0
\end{array}
\right].
\] 
Furthermore, it is easy to see that 
\begin{align*}
& \begin{aligned}
Q\left(A^{(2)}_{2}, A^{(1)}\right) 
&= 
-a Q\left(A^{(2)}_{1}, A^{(1)}\right), \\
Q\left(A^{(2)}_{3}, A^{(1)}\right) 
&= 
-b Q\left(A^{(2)}_{1}, A^{(1)}\right),
   \end{aligned}
\end{align*}
so it clearly suffices to analyze
%Since these matrices are just scaled versions of \(Q\left(A^{(2)}_{1}, A^{(1)}\right) \), in the remaining analysis, we'll only consider
\(\sdp(Q\left(A^{(2)}_{1}, A^{(1)}\right))\).

%Now, to compute \(\sdp(Q(A_{1}^{(2)}, A^{(1)}))\), we need to compute the maximum of two semidefinite programs: one for
We analyze the cases \(z = +1\) and \(z = -1\) separately.
For each case, we first find a lower bound on the SDP by finding a feasible solution. By duality \citep{boyd2004}, we can compute an upper bound on the SDP by finding a dual feasible solution.

We start with \(z = +1\). The primal problem is
\begin{align}
& \begin{aligned}
\max_{P \in S^{6}} 
&\; \; \; 
\left\langle Q\left(A^{(2)}_{1}, A^{(1)}\right), P \right\rangle \\
\text{s.t.} & \; \; \; P \succeq 0 \\
& \; \; \; \left\langle W^{(i)}, P\right\rangle \leq 1 \text{ for } i = 1, \ldots, 6,
   \end{aligned}
   \label{eqnSDPPrimalPlus}
\tag{P+}
\end{align}
where the matrices \(W^{(i)}\) each have the single nonzero entry \(W^{(i)}_{ii} = 1\) and \(S^{6}\) is the space of \(6 \times 6\) symmetric matrices.
Taking the feasible solution \(P = J = \ind \ind ^{\trans}\) gives the lower bound
\[
\left\langle Q\left(A^{(2)}_{1}, A^{(1)}\right), J \right\rangle
=
264.
\]
%Thus, we have the desired lower bound.

The dual program is given by
\begin{align}
& \begin{aligned}
\min_{y \in \reals^{6}_+, L \in S^{6}} 
&\; \; \; 
\left\langle \ind, y \right\rangle \\
\text{s.t.} & \; \; \; L \succeq 0 \\
& \; \; \; L = \sum_{i = 1}^{6} y_{i} W^{(i)} - Q. 
   \end{aligned}
   \label{eqnSDPDualPlus}
\tag{D+}
\end{align}
Consider the feasible pair \((y, L)\) defined by
\[
L
=
\left[
\begin{array}{cccccc}
66 & 0 & 0 & 0 & -6 & -60 \\
0 & 11 & 0 & 0 & -1 & -10 \\
0 & 0 & 22 & 0 & -2 & -20 \\
0 & 0 & 0 & 33 & -3 & -30 \\
-6 & -1 & -2 & -3 & 12 & 0 \\
-60 & -10 & -20 & -30 & 0 & 120 
\end{array}
\right]
\]
%One can observe that this matrix is a Laplacian matrix (see, e.g., \citealp{godsil2001}) and is therefore positive semidefinite.
and \(y = (66, 11, 22, 33, 12, 120)^\trans\), which leads to a dual program value of \(264\). Note that $L \succeq 0$ because it is a diagonally dominant matrix with nonnegative diagonal entries.
Thus, we may conclude that the semidefinite program \eqref{eqnSDPPrimalPlus} must be precisely equal to \(264\).

Next, we consider the case \(z = -1\).
This leads to the primal problem
\begin{align}
& \begin{aligned}
\max_{P \in S^{6}} 
&\; \; \; 
\left\langle -Q\left(A^{(2)}_{1}, A^{(1)}\right), P \right\rangle \\
\text{s.t.} & \; \; \; P \succeq 0 \\
& \; \; \; \left\langle W^{(i)}, P\right\rangle \leq 1 \text{ for } i = 1, \ldots, 6.
   \end{aligned}
   \label{eqnSDPPrimalMinus}
\tag{P-}
\end{align}
In this case, picking \(P = vv^{\trans}\), where 
\(v = (-1, -1, -1, -1, 1, 1)^\trans\), shows that the problem \eqref{eqnSDPPrimalMinus} has a value of at least \(264\).
The dual program is 
\begin{align}
& \begin{aligned}
\min_{y \in \reals^{6}_+, L \in S^{6}} 
&\; \; \; 
\left\langle \ind, y \right\rangle \\
\text{s.t.} & \; \; \; L \succeq 0 \\
& \; \; \; L = \sum_{i = 1}^{6} y_{i} W^{(i)} + Q. 
   \end{aligned}
   \label{eqnSDPDualMinus}
\tag{D-}
\end{align}
The matrix
\[
L
=
\left[
\begin{array}{cccccc}
66 & 0 & 0 & 0 & 6 & 60 \\
0 & 11 & 0 & 0 & 1 & 10 \\
0 & 0 & 22 & 0 & 2 & 20 \\
0 & 0 & 0 & 33 & 3 & 30 \\
6 & 1 & 2 & 3 & 12 & 0 \\
60 & 10 & 20 & 30 & 0 & 120 
\end{array}
\right]
\]
is again a diagonally dominant matrix with nonnegative diagonal entries, so $L \succeq 0$. Furthermore, we can again take \(y = (66, 11, 22, 33, 12, 120)^\trans\) to make $(y, L)$ a feasible solution, implying that the value of \eqref{eqnSDPDualMinus} is upper-bounded by \(264\).
Thus, the value of the semidefinite program \eqref{eqnSDPPrimalMinus} is also equal to \(264\), establishing the lemma.
\end{proof}

%---------------------------------------------%
%---------------------------------------------%
%Now that we have established the value of \(\sdp\left(Q\left(A^{(2)}_{1}, A^{(1)} \right) \right) \), we can prove Proposition~\ref{propIncomparable}.

\begin{proof}[Proof of Proposition~\ref{propIncomparable}]
We start with the inequality \(\lmargin(Tf, z) < \hat{\ell}(f, z)\).
Consider the input vector \(x = (c, c, c)\), for some $c > 0$, and suppose \(\tilde{y} = 1\). Then
\[
f(x) 
= 66c 
\left[ 
\begin{array}{c}
1  \\
-a \\
-b
\end{array}
\right],
\] 
and \(m_{f}(z) = 66c(1 + a)\).
Thus, the loss is
\begin{equation}
\label{EqnEllHat}
\hat{\ell}(f, z) = \phi_{\rho}\left(66c(1 + a) - 132b \epsilon\right).
\end{equation}
For a sufficiently large choice of \(b\), the argument of $\phi_\rho$ becomes negative, so the loss is equal to \(1\).

Next, we calculate \(\lmargin(Tf, z)\).
Assuming that \(\epsilon < c\), we may obtain
\[
Tf(x, y)
=
66
\left[ 
\begin{array}{c}
c - \epsilon  \\
-a(c - \epsilon) \\
-b(c - \epsilon)
\end{array}
\right].
\]
Thus, the loss becomes
\begin{equation}
\label{EqnEllMargin}
\lmargin(Tf, z)
= \phi_\rho \left(m_{Tf}(z)\right) =
\phi_{\rho} \left(66(1 + a)( c - \epsilon)\right).
\end{equation}
Note that we can make the argument larger than $\rho$ by choosing $c$ sufficiently large, in which case $\lmargin(Tf,z) = 0$. This shows the existence of a pair $(f, z)$ such that \(\lmargin(Tf, z) \leq \hat{\ell}(f, z)\).

Turning to the inequality
\(\hat{\ell}(f', z') < \lmargin(Tf', z')\), suppose
we take \(a = 0.5\) and \(b = 0.6\) in the network~\eqref{eqnfExample} to define $f'$.
By equations~\eqref{EqnEllHat} and~\eqref{EqnEllMargin}, we then have
\[
\hat{\ell}(f', z')
=
\phi_{\rho}\left(66c(1 + a) - (1.2)66 \epsilon\right)
\]
and
\[
\lmargin(Tf', z')
=
\phi_{\rho}\left(66c(1 + a) - (1.5)66 \epsilon\right).
\]
Since \(\phi_{\rho}\) is monotonically decreasing, we have \(\hat{\ell}(f', z') \leq \lmargin(Tf', z)\), and we could make the inequality strict by choosing \(c\) such that
\begin{equation*}
0 \le 66c (1+a) - (1.2) 66\epsilon \le \rho.
\end{equation*}
This completes the proof.
\end{proof}

%---------------------------------------------%
%---------------------------------------------%
\section{Adversarial versus distributional robustness}
\label{appAdvDist}

In this appendix, we discuss the relationship between adversarial and distributional robustness.
Specifically, we see that Wasserstein distributional robustness of the kind usually considered is a stronger notion than adversarial robustness.

%---------------------------------------------%
%---------------------------------------------%
\subsection{Definitions}
\label{subsecDefs}

%To articulate differences and similarities between adversarial and distributional robustness,
%we need some vocabulary from the latter.
%We start with defining the Wasserstein distance.
Let \(P\) and \(Q\) be probability measures over \(\reals^{d}\), and let \(\couplings(P, Q)\) denote the set of all couplings of \(P\) and \(Q\). In more detail, if \(P\) and \(Q\) are probability measures defined over the $\sigma$-field $\sfield$,
%of the probability spaces
%\((\reals^{d}, \sfield, P)\) and \((\reals^{d}, \sfield, Q)\) respectively.
a probability measure \(\mu: \sfield \times \sfield \to [0, 1]\) is an element of \(\couplings(P, Q)\) if 
for any event \(A\) in \(\sfield\), we have \(\mu(A, \reals^{d}) = P(A)\) and \(\mu(\reals^{d}, A) = Q(A)\).
%Thus, if \((x, x')\) has distribution \(\mu\), then the marginal distributions of \(x\) and \(x'\) are \(P\) and \(Q\) respectively.
Given a metric \(\metric(\cdot, \cdot)\) on \(\reals^{d}\) and \(1 \le s \le \infty\), the Wasserstein distance is defined as
\begin{align*}
&\begin{aligned}
\wdist_{s}(P, Q)  
&=
\begin{cases}
\inf_{\mu \in \couplings(P, Q)} \expect_{(z, z') \sim \mu} \left[\metric(z, z')^{s}\right]^{\frac{1}{s}}, & s < \infty \\
\essup \metric(z, z'), & s = \infty,
\end{cases}
%\label{eqnWasserstein}
\end{aligned}
\end{align*}
where \(\essup f\) denotes the essential supremum of \(f\). We denote the set of distributions within an \(s\)-Wasserstein distance of \(P\) by
\begin{equation*}
\distset(\dist, \epsilon, s)
=
\left\{Q: W_{s}(\dist, Q) \leq \epsilon\right\}.
%\label{eqnDistSet}
\end{equation*}
The goal in distributionally robust learning
%under perturbations in the Wasserstein distance has been considered before, although usually with \(s = 1\).
is to control a worst-case risk of the form
\begin{equation}
\label{EqnDRO}
\sup_{Q \in \distset(\dist, \epsilon, s)} \expect_{z \sim Q}\left[\ell(f, z)\right],
\end{equation}
where we take \(\dist\) to be the true distribution in our discussion.

%---------------------------------------------%
%---------------------------------------------%
\subsection{Two simple relations}
\label{subsecDistRobust}

%One of our claims that is borne out through computing the perturbation empirical risk and perturbation Rademacher complexity is that adversarial robustness is a less stringent condition than Wasserstein distributional robustness and that this affects the regularization terms.
%However, we can also see this difference directly through two simple lemmas.

We now rigorously derive the fact that the distributionally robust risk upper-bounds the adversarial risk studied in this paper. Thus, adversarial robustness is a less stringent condition than Wasserstein distributional robustness, which is also reflected in the regularization terms appearing in our bounds. We start by showing an equivalence between adversarial robustness and distributional robustness in the case $s = \infty$.

\begin{lemma}
Let \(P\) be a distribution.
Suppose \(\ell\) is continuous or takes finitely many values.
Then
\[
\expect_{P} \left[\sup_{w \in B(\epsilon)} \ell(f, z + w)\right]
=
\sup_{Q \in \distset(P, \epsilon, \infty)} \expect_{z' \sim Q} \, \ell(f, z').
\]
\label{lemmaSkolemize}
\end{lemma}

%---------------------------------------------%
%---------------------------------------------%
\begin{proof}
Let \(\ell\) and \(f\) be given.
We start by proving that
\[
\expect_{P} \left[\sup_{w \in B(\epsilon)} \ell(f, z + w)\right]
\leq
\sup_{Q \in \distset(P, \epsilon)} \expect_{Q} \, \ell(f, z').
\]
Let \(w^{*}\) be a random variable maximizing the supremum on the left-hand side. Since \(\ell(f, \cdot)\) is either continuous or takes finitely many values and \(B(\epsilon)\) is compact, such a random variable exists.

Define \(Q\) such that when \(z' \sim Q\), we have \(z' = z + w^{*}\). Since
\[
\expect_{P} \left[\sup_{w \in B(\epsilon)} \ell(f, z + w)\right]
=
\expect_{Q} \, \ell(f, z'),
\]
we only need to prove that \(Q\) is in \(\distset(P, \epsilon)\). Since we have
\begin{align*}
& \begin{aligned}
\essup \metric(z, z')
&=
\essup \|x - x'\|_{p} 
=
\essup \|w^{*}\|_{p} 
\leq 
\epsilon,
\end{aligned}
\end{align*}
this completes the first direction.

We now prove the reverse inequality:
\[
\expect_{P} \left[\sup_{w \in B(\epsilon)} \ell(f, z + w)\right]
\geq
\sup_{Q \in \distset(P, \epsilon)} \expect_{Q} \, \ell(f, z').
\]
Let \(Q\) be an element of \(\distset(P, \epsilon)\).
Then we can find a sequence of couplings \(\{\mu_{k}\}_{k = 1}^{\infty}\) such that 
when \((z_{k}, z'_{k}) \sim \mu_{k}\), we have
\[
\max \|x'_{k} - x_{k}\|_p
=
\max \metric(z_{k}, z'_{k}) 
\leq 
\epsilon + \frac{1}{k}.
\]
Define \(w_{k} = x'_{k} - x_{k}\). Since all the \(w_{k}\)'s are elements of the compact ball \(B(\epsilon + 1)\), there is a subsequence  \(w_{k_{j}}\) that converges almost surely to some \(w_{\infty}\). Moreover, we see that \(\|w_{\infty}\|_{p} \leq \epsilon\), so \(w_{\infty}\) is always in \(B(\epsilon)\). Denote the limiting measure by \(\mu_{\infty}\). We then have
\begin{align*}
&\begin{aligned}
\expect_{Q} \ell(f, z')
&=
\expect_{\mu_{\infty}} \ell(f, z + w_{\infty}) 
\leq 
\expect_{\dist} \left[\sup_{w \in B(\epsilon)} \ell(f, z + w)\right].
\end{aligned}
\end{align*}
In particular, taking a supremum over $Q \in \distset(P,\epsilon)$ on the left-hand side proves the desired inequality.
\end{proof}

%---------------------------------------------%
%---------------------------------------------%

The second lemma simply states that the robust risk under the \(W_{\infty}\) distance is bounded by the robust risk under the \(W_{s}\) distance for \(s < \infty\).

\begin{lemma}
Let \(P\) be a distribution.
Then for any \(s\) in \([1, \infty]\), we have 
\[
\sup_{Q \in \distset(\dist, \epsilon, \infty)} \expect_{Q} \, \ell(f, z')
\leq 
\sup_{Q \in \distset(\dist, \epsilon, s)} \expect_{Q} \, \ell(f, z').
\]
\label{lemmaMCDistWass}
\end{lemma}

%---------------------------------------------%
%---------------------------------------------%
\begin{proof}
It suffices to show that \(\wdist_{\infty}(P, Q) \leq \epsilon\) implies \(\wdist_{s}(P, Q) \leq \epsilon\). If \(s = \infty\), the inequality is trivial, so assume \(s < \infty\). If \(P\) and \(Q\) are distributions such that \(\wdist_{\infty}(P, Q) \leq \epsilon\), we can find a sequence of couplings \(\{\mu_{k}\}_{k = 1}^{\infty}\) such that,
for \((z_{k}, z'_{k}) \sim \mu_{k}\), we have
\[
\essup \metric(z_{k}, z'_{k})
\leq 
\epsilon + \frac{1}{k}.
\]
As a result, we have
\begin{align*}
& \begin{aligned}
\wdist_{s}(P, Q)
&\leq 
\expect_{\mu_{k}} \left[\metric(z_{k}, z'_{k})^{s}\right]^{\frac{1}{s}} 
\leq 
\essup \left[\metric(z_{k}, z'_{k})^{s} \right]^{\frac{1}{s}} 
\leq 
\epsilon + \frac{1}{k}.
\end{aligned}
\end{align*}
Taking the limit as \(k \to \infty\) shows that \(\wdist_s(P, Q) \leq \epsilon\), proving the lemma.
\end{proof}

\end{document}